\documentclass{article}
\usepackage{bbm}
\usepackage{caption}
\usepackage{rotating}
\usepackage[maxfloats=100]{morefloats}
\usepackage{listings}
\usepackage{booktabs}
\usepackage{xcolor}
\usepackage[title]{appendix}
\usepackage{makecell}
\usepackage{multirow}
\usepackage[symbol]{footmisc}
\usepackage{longtable}
\usepackage{subcaption}
\usepackage[numbers]{natbib}
\usepackage{geometry}
\usepackage{verbatim}
\usepackage{graphicx}
\usepackage{enumitem}
\usepackage{comment}
\usepackage{subfiles}
\usepackage{todonotes}
\usepackage{tikz-qtree,tikz-qtree-compat}
\usepackage{scrextend}
\usepackage{framed}
\usepackage{placeins}
\usepackage{setspace}
\usepackage[ruled,vlined]{algorithm2e}
\usepackage{mathtools,collcell,eqparbox}
\usepackage{multicol}
\usepackage{pgfplots}
\usepackage{tikz}
\usepackage{courier}
\usepackage{url}
\usepackage{color}
\usepackage{authblk}
\usepackage{amssymb}
\usepackage{amsthm}
\usepackage{dsfont}
\usepackage{amsmath}
\usepackage[pagebackref = true]{hyperref}

\hypersetup{
  colorlinks = true,
  linkcolor = teal,
  anchorcolor = teal,
  citecolor = teal,
  filecolor = teal,
  urlcolor = teal
}

\pgfplotsset{compat=1.15}
\numberwithin{equation}{section}

\makeatletter
\newcounter{BMatrix}

\newcommand{\setmaxwd}[1]{%
  \eqmakebox[BM-\theBMatrix][\BMalign]{$#1$}%
}
\MHInternalSyntaxOn

\MHInternalSyntaxOff
\makeatother

\newtheorem{theorem}{Theorem}
\newtheorem{assumption}{Assumption}

\newtheorem{proposition}{Proposition}[section]

\newtheorem{corollary}{Corollary}[theorem]
\newtheorem{lemma}{Lemma}
\theoremstyle{definition}
\newtheorem{definition}{Definition}

\theoremstyle{remark}
\newtheorem{remark}{Remark}
\newtheorem*{remark*}{Remark}

\newcommand{\inv}{^{-1}}

\newcommand{\1}{\mathds{1}}

\newcommand{\R}{\mathbb{R}}

\newcommand{\Z}{\mathbb{Z}}

\newcommand{\E}{\mathbb{E}}

\newcommand{\var}{\mathrm{Var}}

\newcommand{\ra}{\rightarrow}

\newcommand{\cd}{\cdot}
\newcommand{\ds}{\dots}

\newcommand{\bnl}{\mathrm{Bernoulli}}

\newcommand{\mrm}[1]{\mathrm{#1}}

\newcommand{\KL}{\mrm{KL}}

\newcommand{\cC}{\mathcal{C}}

\newcommand{\cF}{\mathcal{F}}
\newcommand{\cG}{\mathcal{G}}
\newcommand{\cH}{\mathcal{H}}

\newcommand{\cJ}{\mathcal{J}}

\newcommand{\cL}{\mathcal{L}}
\newcommand{\cM}{\mathcal{M}}

\newcommand{\cP}{\mathcal{P}}

\newcommand{\cT}{\mathcal{T}}

\newcommand{\tmix}{t_{\mathrm{mix}}}

\newcommand{\spnorm}[1]{\left|#1\right|_\mathrm{span}}
\newcommand{\set}[1]{\left\{{#1}\right\}}
\newcommand{\setcond}[2]{\left\{ #1 : #2 \right\}}

\newcommand{\ceil}[1]{\left\lceil{#1}\right\rceil}

\newcommand{\norm}[1]{\left\|#1\right\|}

\newcommand{\norminf}[1]{\left\|#1\right\|_{\infty}}
\newcommand{\normTV}[1]{\left\|#1\right\|_{\mathrm{TV}}}
\newcommand{\abs}[1]{\left|#1\right|}
\newcommand{\sqbk}[1]{\left[ #1 \right]}

\newcommand{\crbk}[1]{\left( #1 \right)}

\newcommand{\crbkcond}[2]{\left( #1 \middle| #2 \right)}

\newcommand{\bmx}[1]{\begin{bmatrix} #1 \end{bmatrix}}

\newcommand{\bd}[1]{\mathbf{#1}}

\newcommand{\argmax}[1]{\underset{#1}{\operatorname{arg} \operatorname{max}}\;}

\title{Optimal Sample Complexity of \\
Stable Discounted Markov Decision Processes}
\author[1]{Shengbo Wang}
\author[2]{Jose Blanchet}
\author[2]{Peter Glynn}
\affil[1]{Daniel J. Epstein Department of Industrial and Systems Engineering\\ University of Southern California}
\affil[2]{Management Science and Engineering\\ Stanford University}

\begin{document}
\maketitle

\begin{abstract}
We study the optimal sample complexity of tabular reinforcement learning for infinite-horizon discounted Markov decision processes. The unrestricted minimax sample complexity is known to scale as $\widetilde\Theta((1-\gamma)^{-3}\epsilon^{-2})$, where $\gamma$ denotes the discount factor and $\epsilon$ is the solution-error tolerance. However, this worst-case rate does not account for stability structures that commonly arise in operational environments. Using the variance-reduced Q-learning framework, we show that, without imposing any additional stability assumptions, both sup-norm Q-function estimation and policy learning have leading sample-complexity dependence
$\widetilde\Theta(H(1-\gamma)^{-2}\epsilon^{-2})$,
where $H=\spnorm{v^*}$ is the span of the optimal value function. Moreover, assuming a uniform mixing time upper bound $\tmix$ over all policies, the optimal Q-function can be estimated up to a constant shift with sample complexity
$\widetilde\Theta\left(\tmix^3\epsilon^{-2}\right)$, independent of $(1-\gamma)^{-1}$. Matching lower bounds establish the sharpness of these leading dependencies and reveal a separation between estimation and control: although the Q-function can be learned up to an additive constant at a horizon-free rate, policy learning generally retains its $(1-\gamma)^{-2}$ dependence.
\end{abstract}

\section{Introduction}
\label{section:introduction}

Reinforcement learning (RL) has achieved substantial empirical success across a broad range of sequential decision-making problems. Representative applications include the robotic systems surveyed by \citet{Kober2013RLinRoboticsSurvey}, vision-based control studied by \citet{Sadeghi2016Cad2rl}, and financial decision-making studied by \citet{Deng2017rlFinanceSignal}. These developments have generated sustained interest in the statistical and computational foundations of RL, including its application in operations research and management science. This paper contributes to this line of research by studying model-free learning in finite-state, finite-action Markov decision processes (MDPs), where a controller seeks to maximize an infinite-horizon discounted reward.

In many applications, engineering and managerial environments naturally produce MDP models in which the well-performing policies induce system stability; see, for example, \citet{Bramson2008StabilityQ}. Under such a policy, empirical time-averaged performance metrics converge to stationary averages, making steady-state service levels, operating costs, and other long-run performance measures well-defined and insensitive to initial system state. Moreover, the induced Markov chain often converges in distribution to a unique stationary behavior, independently of its initial state. The rate of this convergence is quantified by the mixing time. When a policy mixes rapidly, finite samples from the state process provide a more accurate statistical description of its long-run behavior. It is therefore natural to expect that policy learning should require fewer samples in systems with favorable mixing properties. This intuition is the principal motivation for the paper. 

The minimax theory of discounted tabular RL is now well developed~\citep{azar2013,sidford2018near_opt,agarwal2020,wainwright2019,li2021QL_minmax}. What remains less understood is how general stability quantities affect the sample complexity and the design of RL algorithms, and how this dependence varies across learning objectives, such as sup-norm Q-function estimation, policy learning, and Q-function estimation up to an additive constant. In addition, a stability-sensitive theory for discounted MDPs is closely connected to the theory of average-reward RL. Discounted MDPs with a sufficiently large discount factor provide an important route for approximating average-reward MDPs and learning near-optimal policies, as first demonstrated by \citet{jin_sidford2021}. Leveraging this discounted approximation technique, recent work has substantially strengthened this connection and obtained optimal results \citep{wang2024average_rwd,zurek2024span}. This further highlights the importance of developing a stability-structure-aware theory.

This paper explores how stability structure changes the statistical difficulty of discounted tabular RL. In discounted MDPs, the discount factor $\gamma$ induces an effective horizon $(1-\gamma)^{-1}$ that controls how far information is propagated. Assume access to a generative model, worst-case discounted RL theory shows that the minimax sample complexity to achieve $\epsilon$-error in optimal value estimation and policy learning is then
$\widetilde\Theta(|S||A|{(1-\gamma)^{-3}\epsilon^{-2}})$
~\citep{azar2013,sidford2018near_opt,agarwal2020}, where $|S|$ and $|A|$ are the cardinality of the state and action space, respectively. However, these worst-case results do not account for additional stability structures that are common in OR applications. This leaves open a more structure-aware theory: when the controlled system exhibits favorable transient or ergodic behavior, which aspects of the worst-case horizon dependence are genuinely necessary?

We study this question using a model-free, value-based approach by adapting the variance-reduced Q-learning (VRQL) algorithm of \citet{wainwright2019} under a generative model. Model-based methods estimate a transition model and solve the resulting empirical dynamic program. Model-free methods instead update value or action-value estimates directly from sampled Bellman operators. This distinction matters beyond the tabular setting. A complete transition model is expensive to store and repeatedly optimize over in large state spaces, whereas value-based updates are naturally combined with approximation architectures. Classical neuro-dynamic programming and Q-learning already embody this principle, and deep Q-learning extends the same architecture to high-dimensional observations~\citep{Watkins1992,bertsekas_tsitsiklis1996,mnih2015deep_q}.

The central structural quantity in our discounted analysis is $H:=\spnorm{v^*},$ the span of the optimal value function. Unlike classical stability assumptions such as communication, recurrence, or mixing, the finiteness of $H$ requires no additional assumption in a finite discounted MDP: $[0,1]$-bounded rewards imply $H\leq(1-\gamma)^{-1}$. More importantly, the one-sample Bellman variance is invariant to adding a constant to $v^*$, so its intrinsic magnitude is governed by $H$, rather than by the potentially much larger level $\norminf{v^*}$. Thus, $H$ provides a quantitative measure of transient variation even when the MDP is otherwise outside the usual stability classes. Moreover, it is smaller than other commonly used stability metrics, such as the mixing time, the hitting time of a recurrent state, and the diameter of the MDP, thereby providing a tighter characterization.

Beyond analyzing the minimax optimal sample complexity of Q-estimation and policy learning under the sup norm, we also study Q-estimation when the error is measured in the span seminorm. Since an optimal greedy policy can be identified even when the optimal Q-function is known only up to an additive constant, it is natural to ask whether learning the Q-function up to a constant can be statistically easier. As discussed later, we show that under uniform mixing, estimating the Q-function in span error indeed yields an improved sample complexity whose leading term is independent of the effective horizon. This effect is also observed by \citet{devraj2021q} in their asymptotic analysis of relative Q-learning. However, our lower bound results show that this improvement in Q-estimation does not translate into a faster policy learning rate in a minimax sense.

\paragraph{Main results and implications.}
Table~\ref{tab:sample_complexity} summarizes the precise leading sample complexity dependencies. We discuss some important implications of the three upper- and lower-bound results.

\begin{enumerate}
\item \textit{Q-function estimation in the sup norm.} Without imposing stability conditions, VRQL adapts to the span $H$. Relative to the unrestricted worst-case rate, this result replaces one factor of the effective horizon $(1-\gamma)^{-1}$ with the intrinsic transient scale $H$. The matching lower bound shows that this improvement is sharp over a recurrent family of MDPs for which $H\asymp\tmix$. In particular, the sample complexity upper bound remains minimax optimal even when the problem class is restricted to recurrent MDPs.

\item \textit{Q-function estimation in the span seminorm.} If every policy induces a mixing Markov chain, VRQL can learn the optimal Q-function up to an additive constant at a rate whose leading statistical term is independent of $(1-\gamma)^{-1}$. This faster rate is also minimax optimal.

\item \textit{Policy learning.} The same leading stability-driven improvement established for sup-norm Q-function estimation extends to the greedy policy induced by the final Q-iterate. The sup-norm Q-estimation guarantee can therefore be converted into an equally sharp policy guarantee without imposing any additional structural condition on the MDP. The corresponding lower bound is also established over the same recurrent MDP family. Consequently, the $(1-\gamma)^{-2}$ dependence generally cannot be removed from policy learning, even though the Q-function can be estimated up to an additive constant at a rate whose leading term is independent of the effective horizon $(1-\gamma)^{-1}$.
\end{enumerate}

\begin{table}[t]
\centering
\small
\begin{tabular}{p{0.25\textwidth}p{0.25\textwidth}p{0.27\textwidth}}
\toprule
Learning objectives & VRQL upper bounds  & Lower bounds $(|S||A| = O(1))$ \\
\midrule
Q estimation sup error
& $\widetilde O\!\left(|S||A|H(1-\gamma)^{-2}\epsilon^{-2}\right)$
& $\Omega\!\left(H(1-\gamma)^{-2}\epsilon^{-2}\right)$\\

\addlinespace
Q estimation span error
& $\widetilde O\!\left(|S||A|\tmix^3\epsilon^{-2}\right)$
& $\Omega\!\left(\tmix^3\epsilon^{-2}\right)$ \\

\addlinespace
Policy value loss
& $\widetilde O\!\left(|S||A|H(1-\gamma)^{-2}\epsilon^{-2}\right)$
& $\Omega\!\left(H(1-\gamma)^{-2}\epsilon^{-2}\right)$ \\
\bottomrule
\end{tabular}
\caption{Sample complexity results in this paper. The sup error and policy value loss upper bounds require no structural assumptions. The span error Q-estimation bounds assume that all policies induce mixing Markov chains with mixing times bounded by $\tmix$.}
\label{tab:sample_complexity}
\end{table}

Together, these results separate three effects that are conflated in the unrestricted discounted theory. A small optimal-value span makes ordinary Q-estimation and policy learning easier even without an ergodicity assumption. Uniform mixing yields a further improvement when Q-estimation errors are measured modulo constants. That second improvement, however, does not apply to policy learning. This distinction is particularly relevant for average-reward RL, where bias and relative-Q functions are defined only up to a constant, but the ultimate objective remains the identification of a near-optimal policy.

\subsection{Literature Review}

\paragraph{Discounted tabular RL and value-based model-free learning.}
For discounted MDPs with a generative model, model-based methods achieve the minimax optimal sample complexity $\widetilde\Theta(|S||A|(1-\gamma)^{-3}\epsilon^{-2})$~\citep{azar2013,sidford2018near_opt,agarwal2020,li2022settling} for both Q-function estimation and policy learning. Value-based model-free algorithms have been attracting significant research interest due to their ease of use and compatibility with function approximation frameworks. \citet{li2021QL_minmax} show that standard Q-learning has a suboptimal worst-case horizon dependence $(1-\gamma)^{-4}$, whereas \citet{wainwright2019} propose variance-reduced Q-learning, which recovers the model-based minimax rate while maintaining a model-free recursion. \citet{wainwright2019l_infty} and \citet{chen2020} provide finite-time analyses of contractive stochastic approximation that form the technical foundation for these developments. \citet{li2023PR_averaging_Ql} further obtain refined asymptotic and non-asymptotic behavior through Polyak averaging.

\paragraph{Stability and quantitative parameters.}

Classical MDP theory distinguishes several stability structures; see \citet{Puterman1994}. A communicating MDP allows every state to reach every other under some policy, while weak communication permits additional states that are transient under all stationary policies. A unichain MDP requires every stationary policy to induce a single recurrent class. The reachability condition of \citet{chen2026average_reward_q_learning} requires a common reference state to be reachable from every initial state under every stationary deterministic policy. Recurrence requires the entire state space to be recurrent under any stationary policy. These qualitative structures are quantified in several complementary ways. \citet{jaksch2010near_optimal_regret} use the diameter to measure worst-case controlled hitting times in communicating MDPs. \citet{bartlett2009regal} and \citet{fruit2018scal} use the span of the optimal bias to measure variation in transient performance under weak communication. The uniform mixing time controls how rapidly all policy-induced chains forget their initial states~\citep{Wang2017PrimalDualL,jin_sidford2021,wang2024average_rwd}. \citet{zurek2024span} use a transient-time parameter for general multichain models, while \citet{chen2026average_reward_q_learning} use reference-state hitting times under reachability. These parameters capture different aspects of stability and need not be comparable without additional assumptions.

\paragraph{Stability and Average-reward RL.}
The assumptions imposed in average-reward RL reflect this range of stability structures. Among online methods evaluated by cumulative regret, \citet{jaksch2010near_optimal_regret} propose UCRL-type algorithms for communicating MDPs and express their guarantees through the diameter, while \citet{bartlett2009regal} and \citet{fruit2018scal} propose REGAL and SCAL for weakly communicating MDPs using the span of the optimal bias function. The online model-free algorithms of \citet{wei2020_model_free_average_reward} use optimistic Q-learning under weak communication and obtain a sharper regret guarantee under the stronger ergodic assumption through a policy-optimization method. 

Algorithms for policy, value, and Q learning aim to return an $\epsilon$-optimal policy or an accurate relative value or Q-function after a finite number of samples. For finite-sample policy learning, many early algorithms and analyses assume the stronger condition that all policies induce mixing~\citep{Wang2017PrimalDualL,jin_sidford2021,wang2024average_rwd}. \citet{zurek2024span} and \citet{zurek2025span_agnostic} study weakly communicating MDPs using a finite optimal-bias span, while \citet{zurek2024span} additionally treat general multichain dynamics using a transient-time parameter. \citet{zurek2025plugin} obtain diameter- and mixing-based policy-learning guarantees without requiring prior knowledge of the corresponding stability parameter. 

For value- and Q-learning, the classical convergence analysis of \citet{abounadi2001average_cost} uses unichain-type conditions. More recent finite-sample approaches include the adaptive discounted-Q-learning reductions of \citet{jin2024feasible_q_learning} under mixing or span regularity, the stochastic Halpern and anchored fixed-point methods of \citet{bravo2026halpern} and \citet{lee2025anchoring}, and the synchronous and asynchronous Q-learning algorithms of \citet{chen2026average_reward_q_learning} under a reference-state reachability condition. 

This diversity underscores that average-reward RL is not governed by a single canonical stability assumption: the attainable guarantee depends both on whether the objective is online regret, policy learning, or value/Q estimation and on how stability condition is imposed. 

\paragraph{Structure-adaptive and instance-dependent RL.}
A separate line of work asks when RL can be easier than its global minimax benchmark because of favorable local or decision structure. \citet{simchowitz2019gap} derive gap-dependent regret bounds that exploit separation between optimal and suboptimal actions. \citet{foster2021instance_dependent} and \citet{wagenmaker2022instance_pac} use disagreement and reachability complexities to characterize instance-dependent learning in contextual bandit and RL settings. For discounted value estimation, \citet{khamaru2021} and \citet{Khamaru2021_TD} develop local minimax theories that identify propagated Bellman variance as the intrinsic instance difficulty and motivate adaptive variance-reduced procedures. \citet{xia2023early_stopping} construct data-dependent confidence regions and early-stopping rules that exploit this local difficulty in practice.

\section{Markov Decision Processes and Stability}
\label{section:mdp_stability}
Let $\cM = (S, A, r, P, \gamma)$ be an MDP, where $S$ and $A$ are finite state and action spaces, $r:S\times A\ra [0,1]$ the reward function, $P:S\times A\ra \cP(S)$ the transition kernel, where $\cP(S)$ is the set of probability measures on $S$. $\gamma \in (0,1)$ is the discount factor. We assume that $S,A,r,\gamma$ are known, and we can draw samples from the transition measures $\set{P_{s,a},s\in S,a\in A}$ independently through a sampler (generative model). 
\par Let $\Omega=(S\times A)^{\Z_{\geq0}}$, and let $\cF$ be the product discrete $\sigma$-field on $\Omega$. Define the stochastic process $\set{(S_t,A_t),{t\geq 0}}$ by the point evaluation $S_t(\omega) = s_t,A_t(\omega) = a_t$ for all $t\geq 0$ for any $\omega = (s_0,a_0,s_1,a_1,\ds)\in\Omega$. At time $t\geq 0$ and state $s_t\in  S$, if action $a_t\in A$ is chosen, the decision maker will receive reward $r(s_t,a_t)$, and then the law of the subsequent state is determined by the transition kernel $\cL(S_{t+1} |S_0 = s_0, \ds,S_{t}=s_t,A_t = a_t) = P_{s_t,a_t}(\cdot)$. It is well known that to achieve optimal decision making in the context of infinite horizon discounted MDPs (to be introduced), it suffices to consider policies that are stationary, Markov, and deterministic. Therefore, we will restrict ourselves to consider only this class of policies. 
\par Let $\Pi$ denote the class of stationary Markov deterministic policies; i.e. any $\pi\in \Pi$ can be seen as a function $\pi:S\ra A$. For $\pi\in \Pi$ and initial distribution $\mu\in\cP(S)$, there is a probability measure $\mathbb P_{\mu}^\pi$ on the product space  s.t. the chain $\set{(S_t,A_t),t\geq 0}$ has finite dimensional distributions
\[
\mathbb P^\pi_\mu(S_0 =  s_0, A_0 = a_0,\ds,A_t = a_t) = \mu(s_0)P_{s_0,a_0}(s_1)P_{s_1,a_1}(s_2)\ds P_{s_{t-1},a_{t-1}}(s_t)\1\set{\pi(s_i) = a_i,i\leq t}.
\]
Note that this also implies that $\set{S_t,t\geq 0}$ is a Markov chain under $\mathbb{P}_{\mu}^\pi$. Let $\E_{\mu}^\pi$ denote expectation under $\mathbb{P}_{\mu}^\pi$. When $\mu$ is a point-mass $\delta_s$ at $s\in S$, we write $\E_s^\pi = \E_{\delta_s}^\pi$. 

The value function $v^{\pi}(s)$ is defined via
\[
v^{\pi}(s) := \E_s^{\pi}\sqbk{\sum_{t=0}^\infty \gamma^{t}r(S_t,A_t) },
\]
while the optimal value function is given by
\begin{equation}\label{eqn:opt_val_func}
v^*(s) := \max_{\pi\in\Pi} v^\pi(s), \quad \forall s\in S.
\end{equation}
It can be seen as a vector $v^*\in\R^{|S|}$. Let $\Pi^*$ denote the set of all optimal stationary Markov deterministic policies; i.e. $\pi\in\Pi$ that achieves the maximum in \eqref{eqn:opt_val_func}. 
\par Let $P_{s,a}[v]$ denote the sum $\sum_{s'\in S}P_{s,a}(s')v(s')$. It is well known that the optimal value function is the unique solution of the following Bellman equation:
\begin{equation}\label{eqn:Bellman_eqn_v}
v^*(s) = \max_{a\in A}\crbk{r(s,a) + \gamma P_{s,a}[v^*]}.
\end{equation}
\par To learn an optimal policy $\pi^*\in\Pi^*$, it is useful to introduce the Q-function: For any policy $\pi\in\Pi$,
\[
q^\pi(s,a) := r(s,a) + \gamma P_{s,a}[v^\pi],
\]
which can be identified with a vector $q^\pi\in \R^{|S| |A|}$. Define the notation $v(q)(s) = \max_{a\in A}q(s,a)$. When $\pi = \pi^*$ achieves the maximum in \eqref{eqn:opt_val_func}, we denote the corresponding optimal Q-function $q^{\pi^*}$ by $q^*$ and hence $v(q^*) = v^*$. Note that $q^*$ is the unique solution to the Bellman equation for the Q-function:
\begin{equation}\label{eqn:Bellman_eqn_q}
\begin{aligned}
q^*(s,a) &=r(s,a) + \gamma \sum_{s'\in S}P_{s,a}(s')\max_{b\in A} q^*(s',b)\\
    &=r(s,a) + \gamma P_{s,a}\sqbk{v( q^*)}\\
    &= \cT(q^*)(s,a)
\end{aligned}
\end{equation}
where the mapping $\cT$ defined by
\begin{equation}\label{eqn:def_q_Bellman_eqn}
\cT(q)(s,a): = r(s,a) + \gamma P_{s,a}\sqbk{v( q)},\quad \forall(s,a) \in S\times A
\end{equation} is known as the Bellman operator for the Q-function. Here we list two elementary properties of the Bellman operator and the Q-function. The proofs can be found in \citet{Puterman1994}. First, $\cT$ is a $\gamma$-contraction in the $\|\cd\|_\infty$-norm; i.e. $\|\cT(q)-\cT(q')\|_\infty\leq \gamma\|q-q'\|$ for all $q,q'\in \R^{|S| |A|}$.
Then, a consequence of $\cT$ being a $\gamma$-contraction is 
$\|q^*\|_\infty\leq 1/(1-\gamma).$

A greedy policy from a Q-function is defined by $\pi(q)(s)\in\arg\max_{a\in A}q(s,a)$, where one action is selected if there is a tie. It is easy to see that $\Pi^*\ni\pi^*=\arg\max_{a\in A}q^*(\cd,a)=\pi(q^*)$. Therefore, policy learning can be achieved if we can learn a good estimate of $q^*$.

For $\pi\in\Pi$, define the state transition kernel $P_\pi$ and the lifted state-action transition kernel $P^\pi$ by
\begin{equation}
\label{eqn:def_policy_kernels}
P_\pi(s,s'):=P_{s,\pi(s)}(s')
\quad\text{and}\quad
P^\pi [q](s,a):=\sum_{s'\in S}P_{s,a}(s')q(s',\pi(s')).
\end{equation}
Notice that, starting from any initial state-action pair, $P^\pi$ selects the next state according to $P_{s,a}(\cdot)$ and then selects the next action according to $\pi$. Thus, $P^\pi$ is the transition kernel of the Markov chain $\set{(S_t,A_t):t\geq 0}$.

Next, we introduce the span seminorm and use the span of the optimal value function as a quantitative stability parameter. We also clarify its relationship with other stability assumptions and parameters commonly used in the MDP literature.

\begin{definition}[Span of the optimal value]
\label{def:span_v}
For $v\in\R^d$, define the span seminorm
\[
\spnorm{v}:=\max_{1\leq i\leq d}v_i-\min_{1\leq i\leq d}v_i.
\]
Let $H:=\spnorm{v^*}$ be the span of the unique optimal value function. 
\end{definition}

Note that $0\le v^* \le\norm {v^*}\leq 1/(1-\gamma)$. So $H\leq 1/(1-\gamma)$. Thus, the finiteness of $H$ requires no communication, unichain, recurrence, or mixing condition; it is automatically bounded for every finite discounted MDP with bounded rewards. When stronger stability structure is available, $H$ is controlled by the corresponding quantitative stability parameters. In particular, if the MDP is communicating with diameter $D$, then $H\lesssim D.$
The corresponding comparisons between the optimal bias span, discounted optimal-value span, and diameter are established by \citet{Wang2022} and \citet{zurek2024span}. Consequently, the $H$-dependent guarantees developed below recover diameter-dependent guarantees for communicating MDPs, while remaining applicable when the diameter is infinite.

\par For probability measures $\mu$ and $\nu$ on a finite set $E$, we use the total-variation convention
\begin{equation}
\label{eqn:def_total_variation_convention}
\normTV{\mu-\nu}:=\sup_{B\subseteq E}\abs{\mu(B)-\nu(B)}=\frac12\sum_{x\in E}\abs{\mu(x)-\nu(x)}.
\end{equation}
\begin{definition}[Total-variation mixing time]
\label{def:mixing_time}
Let $K$ be a Markov kernel on a finite set $E$ with unique stationary distribution $\eta_K$. Its total-variation mixing time is
\begin{equation}
\label{eqn:def_mixing_time}
\tmix(K):=\min\setcond{m\ge 1}{\sup_{x\in E}\normTV{K^m(x,\cd)-\eta_K}\leq\frac14}.
\end{equation}
If the displayed set is empty, we set $\tmix(K)=\infty$.
\end{definition}

If an optimal policy $\pi^*\in\Pi^*$ induces a mixing Markov chain, then the optimal-value span is also controlled by the mixing time of this single policy: $H\lesssim\tmix(P_{\pi^*})$; see \citet{Wang2022} and the related optimal-policy mixing comparison in \citet{zurek2024span}. Hence, the $H$-dependent results below automatically imply optimal-policy mixing-time-dependent guarantees whenever such a policy exists. Importantly, their validity does not require every policy to mix, or even require the existence of any mixing policy.

\section{Variance Reduction Q-learning for Stable MDPs}
\label{section:algos_and_cplx}
\label{section:span_dependent_vrql}

We introduce the variance-reduced Q-learning (VRQL) framework analyzed in this paper. Standard Q-learning performs one stochastic Bellman update at each iteration and is shown by \citet{li2021QL_minmax} to be minimax suboptimal. Wainwright's variance-reduction scheme adds a recentering component to the Q-learning structure, resulting in a two-loop algorithm. \citet{wainwright2019} shows that this approach achieves the minimax-optimal sample complexity for sup-norm Q-estimation.

In this section, we develop span-dependent convergence guarantees for VRQL. The estimation error is measured in the ordinary sup norm. No irreducibility, unichain, mixing, or minorization assumption is imposed in this section. To better understand the variance-reduced algorithm, we first present the classical synchronous Q-learning. 

\subsection{The Q-learning Algorithm}

\providecommand{\ql}{\text{QL}}
For every $k\geq 1$, let $\widehat\cT_k$ be an independent one-sample empirical Bellman operator. Thus, independently for every $(s,a)\in S\times A$, draw $S'_k(s,a)\sim P_{s,a}$ and set
\begin{equation}
\label{eqn:1-sample_Bellman}
\widehat\cT_k(q)(s,a)
:=r(s,a)+\gamma\max_{b\in A}q(S'_k(s,a),b).
\end{equation}
Each one-sample Bellman operator therefore uses $|S||A|$ transition samples from the generative model. The synchronous Q-learning recursion is given in Algorithm~\ref{alg:q-learning}.

\begin{algorithm}[htb]
\caption{Q-learning}
\label{alg:q-learning}

\textbf{Initialization:} $q_1\equiv 0$; $k=1$\;

\For{$1\leq k\leq k_0$}{
    Sample $\widehat\cT_{k+1}$, the one-sample empirical Bellman operator\;

    Compute the Q-learning update
    \[
    q_{k+1}
    =(1-\lambda_k)q_k
    +\lambda_k\widehat\cT_{k+1}(q_k)
    \]
    with stepsize $\lambda_k=1/(1+(1-\gamma)k)$\;
}

\KwRet{$q_{k_0+1}$}\;
\end{algorithm}

To a stability-aware theory, we need a sharp characterization of the algorithm's estimation variance. For any $q\in\R^{|S|\times|A|}$, define the standard deviation of one-sample Bellman operator by
\begin{equation}\label{eqn:1-samp_var}
\sigma(q)(s,a):=\gamma\sqrt{P_{s,a}[v(q)^2]-P_{s,a}[v(q)]^2}.
\end{equation}
We remark that this is the sampling standard deviation of the one-sample Bellman operator \eqref{eqn:1-sample_Bellman}. For compactness, we write $\sigma^*:=\sigma(q^*)$.

We note that stability structure can already improve the convergence rate of Q-learning if we carefully track the error in terms of $H$. The improvement follows from two elementary consequences of the Bellman equation:
\begin{equation}
\label{eqn:q_learning_two_span_bounds}
\norminf{\sigma^*}\leq \frac{\gamma H}{2},
\quad
\spnorm{q^*}\leq 1+\gamma H.
\end{equation}
The first inequality follows because a random variable with range length $H$ has variance at most $H^2/4$. The second follows by comparing
\[
q^*(s,a)=r(s,a)+\gamma P_{s,a}[v^*]
\]
at two arbitrary state-action pairs. Substituting \eqref{eqn:q_learning_two_span_bounds} into the finite-time Q-learning estimate of \citet{wainwright2019l_infty,wainwright2019} gives the following proposition. We place the proof details in the appendix.

\begin{proposition}
[Span-dependent convergence of Q-learning]
\label{prop:q_learning_span_convergence}
There exists a universal constant $C\geq 1$ such that, for every $k_0\geq 1$ and $\delta\in(0,1)$, the output of Algorithm~\ref{alg:q-learning} satisfies, with probability at least $1-\delta$,
\begin{align}
\norminf{q_{k_0+1}-q^*}
&\leq C\left[
\frac{\gamma H}{(1-\gamma)^{3/2}}
\sqrt{\frac{1}{k_0}
\log\crbk{\frac{2e|S||A|k_0(1+(1-\gamma)k_0)}{\delta}}}
\right.\notag\\
&\hspace{3.4cm}\left.
+\frac{1+\gamma H}{(1-\gamma)^2k_0}
\log\crbk{\frac{2e|S||A|k_0(1+(1-\gamma)k_0)}{\delta}}
\right].
\label{eqn:q_learning_span_convergence}
\end{align}    
\end{proposition}

Ignoring logarithmic factors, Proposition~\ref{prop:q_learning_span_convergence} gives the transition-sample bound
\begin{equation}
\label{eqn:q_learning_span_sample_complexity_informal}
\widetilde O\left(
|S||A|\left[
\frac{H^2}{(1-\gamma)^3\epsilon^2}
+\frac{1+H}{(1-\gamma)^2\epsilon}
\right]\right)
\end{equation}
for an estimation error $\norminf{q_{k_0+1}-q^*}\leq\epsilon$. 

We do not claim that \eqref{eqn:q_learning_span_sample_complexity_informal} is optimal for the Q-learning algorithm. Indeed, when $H\asymp(1-\gamma)\inv$, the worst-case analysis of \citet{li2021QL_minmax} gives the better $\widetilde O((1-\gamma)^{-4}\epsilon^{-2})$ dependence, which is optimal for Q-learning. The purpose of Proposition~\ref{prop:q_learning_span_convergence} is instead to expose the direct improvement enjoyed by standard Q-learning when the optimal value has small span. In comparison, the variance-reduced algorithm below improves the leading effective horizon dependence to $(1-\gamma)^{-2}$ and span dependence to $H$, both are minimax optimal.

\subsection{Wainwright's Variance Reduction}

The variance-reduced recursion uses the one-sample empirical Bellman operators defined in \eqref{eqn:1-sample_Bellman}.

The algorithm below is the variance-reduced Q-learning recursion of \citet{wainwright2019}, written with an epoch-dependent inner-loop length $k_l$.  Allowing $k_l$ to depend on the epoch lets us use the usual inner-loop length for estimating $q^*$, while running a longer inner loop only in the final epoch when an accurate greedy policy is required.

\providecommand{\vrql}{\mathrm{VRQL}}
\begin{algorithm}[htb]
\caption{Variance-reduced Q-learning}
\label{alg:vr_q-learning}

\textbf{Initialization:} $\hat q_0\equiv 0$\;

\For{$1\leq l\leq l^*$}{
    Sample i.i.d. $\set{\widetilde\cT_{l,j}:1\leq j\leq n_l}$ one-sample empirical Bellman operators and set
    \[
    \widetilde\cT_l(\hat q_{l-1})
    :=\frac{1}{n_l}\sum_{j=1}^{n_l}\widetilde\cT_{l,j}(\hat q_{l-1}).
    \]

    Set $q_{l,1}=\hat q_{l-1}$\;

    \For{$1\leq k\leq k_l$}{
        Sample an independent one-sample empirical Bellman operator $\widehat\cT_{l,k+1}$ and compute
        \begin{equation}
        \label{eqn:vr_q-learning_update}
        q_{l,k+1}
        =(1-\lambda_k)q_{l,k}
        +\lambda_k\crbk{
        \widehat\cT_{l,k+1}(q_{l,k})
        -\widehat\cT_{l,k+1}(\hat q_{l-1})
        +\widetilde\cT_l(\hat q_{l-1})},
        \end{equation}
        where $\lambda_k=1/(1+(1-\gamma)k)$\;
    }

    Set $\hat q_l=q_{l,k_l+1}$ and
    \[
    \hat\pi_l^*(s)=\argmax{a\in A}\hat q_l(s,a),\quad s\in S.
    \]
}

\KwRet{$\hat q_{l^*}$ and $\hat\pi_{l^*}^*$}
\end{algorithm}

We next describe the fixed point tracked during a single epoch.  Let $\cG_l$ be the sigma-field generated by all samples used through the end of epoch $l$.  Conditional on $\cG_{l-1}$ and the recentering batch in epoch $l$, define
\begin{equation}
\label{eqn:def_recentered_operator_span_dependent}
\begin{aligned}
\widehat\cJ_{l,k+1}(q)
&:=\widehat\cT_{l,k+1}(q)-\widehat\cT_{l,k+1}(\hat q_{l-1})
+\widetilde\cT_l(\hat q_{l-1}),\\
\cJ_l(q)
&:=\cT(q)-\cT(\hat q_{l-1})+\widetilde\cT_l(\hat q_{l-1}).
\end{aligned}
\end{equation}
The operator $\cJ_l$ is a $\gamma$-contraction and is the conditional expectation of $\widehat\cJ_{l,k+1}$.  Its unique fixed point is denoted by $\bar q_l$.  Equivalently, define the perturbed reward
\begin{equation}
\label{eqn:def_bar_r_l}
\bar r_l:=r-\cT(\hat q_{l-1})+\widetilde\cT_l(\hat q_{l-1}).
\end{equation}
Then $\bar q_l$ is the optimal action-value function of the MDP with reward $\bar r_l$ and the original transition kernel:
\begin{equation}
\label{eqn:def_bar_q_l}
\bar q_l(s,a)=\bar r_l(s,a)+\gamma P_{s,a}[v(\bar q_l)].
\end{equation}
Let $\bar\pi_l$ be an optimal policy for this perturbed MDP.

\subsection{Convergence Behavior and Sample Complexities}

The next theorem states the two finite-time convergence guarantees for VRQL.  The first part uses the standard VRQL inner-loop length and gives simultaneous geometric convergence of all outer iterates.  The second part changes only the number of inner iterations in the final epoch; this additional final-epoch computation is what converts an accurate action-value estimate into an equally accurate greedy policy.

\begin{theorem}[Q-estimation and policy error convergence]
\label{thm:pathwise_q_and_policy_error}
There exists a universal constant $C\geq 1$ such that the following statements hold.  Fix $l^*\geq 1$ and $\delta\in(0,1)$.  For $1\leq l\leq l^*$, choose recentering sample sizes
\begin{equation}
\label{eqn:n_l_span_dependent_choice}
n_l
=\ceil{
C\crbk{4^lH+\frac{1}{(1-\gamma)^2}}
\log\crbk{\frac{8|S||A|l^*}{\delta}}},
\end{equation}
and inner loop iterations
\begin{equation}
\label{eqn:k_q_k_policy_choice}
k^*=\ceil{
\frac{C}{(1-\gamma)^3}
\log\crbk{\frac{8|S||A|l^*}{(1-\gamma)\delta}}}\quad\text{and}\quad
k^*_{\mathrm{policy}}=\ceil{
\frac{C}{(1-\gamma)^5}
\log\crbk{\frac{8|S||A|l^*}{(1-\gamma)\delta}}}.
\end{equation}

\begin{enumerate}[label=\textnormal{(\roman*)}]
\item Run Algorithm~\ref{alg:vr_q-learning} with $k_l=k^*$ for every $1\leq l\leq l^*$.  Then, with probability at least $1-\delta$,
\begin{equation}
\label{eqn:pathwise_q_error_standard}
\norminf{\hat q_l-q^*}
\leq \frac{2^{-l}}{1-\gamma},
\quad \forall \,0\leq l\leq l^*.
\end{equation}

\item Run the same algorithm with $k_l=k^*$ for $1\leq l<l^*$ and $k_{l^*}=k^*_{\mathrm{policy}}$ in the final epoch.  Then, with probability at least $1-\delta$, the pathwise bound \eqref{eqn:pathwise_q_error_standard} continues to hold and the final greedy policy satisfies
\begin{equation}
\label{eqn:pathwise_policy_error}
\norminf{v^*-v^{\hat\pi_{l^*}^*}}
\leq \frac{2^{-l^*}}{1-\gamma},
\quad\text{and}\quad
\norminf{q^*-q^{\hat\pi_{l^*}^*}}
\leq \frac{2^{-l^*}}{1-\gamma}.
\end{equation}
\end{enumerate}
\end{theorem}
The proof of Theorem \ref{thm:pathwise_q_and_policy_error} is presented in Section \ref{sec:proof:thm:pathwise_q_and_policy_error}. 

We note that, unlike the original VRQL designs in \citet{wainwright2019} and \citet{khamaru2021}, Theorem~\ref{thm:pathwise_q_and_policy_error} and the later Theorem \ref{thm:pathwise_span_error_uniform_mixing} do not require a Q-learning warm start or restarting procedure. This is because our batch-size schedule includes an additional $\widetilde O(1/(1-\gamma)^2)$ term in $n_l$ used to recenter the Bellman operator. This additional floor ensures that the initial epochs satisfy the same contraction target without a warm start. More generally, the same conclusion holds for any initialization satisfying $\norminf{\hat q_0-q^*}\leq 1/(1-\gamma)$.

Theorem \ref{thm:pathwise_q_and_policy_error} then leads to the following sample complexity guarantee. 

\begin{corollary}[Sample complexities]
\label{cor:span_dependent_sample_complexity}
Let $\epsilon\in(0,(1-\gamma)\inv)$ and $\delta\in(0,1)$, and set
\[
l_\epsilon
:=\ceil{\log_2\crbk{\frac{1}{(1-\gamma)\epsilon}}}.
\]
There is a universal constant $C\geq 1$ such that the following conclusions hold when the parameters in Theorem~\ref{thm:pathwise_q_and_policy_error} are chosen with $l^*=l_\epsilon$.

\begin{enumerate}[label=\textnormal{(\roman*)}]
\item The standard implementation produces $\hat q_{l_\epsilon}$ satisfying $\norminf{\hat q_{l_\epsilon}-q^*}\leq\epsilon$
with probability at least $1-\delta$, using at most
\begin{equation}
\label{eqn:exact_q_sample_complexity}
C|S||A|
\crbk{
\frac{H}{(1-\gamma)^2\epsilon^2}
+\frac{l_\epsilon}{(1-\gamma)^3}}
\log\crbk{
\frac{8|S||A|l_\epsilon}{(1-\gamma)\delta}}
\end{equation}
transition samples. 
\item The policy implementation with $k_{l_{\epsilon}} = k^*_\mathrm{policy}$ produces the $\hat q_{l_\epsilon}$-greedy policy $\hat \pi:= \hat\pi_{l_\epsilon}^*$ satisfying
$\norminf{v^*-v^{\hat\pi}}\leq\epsilon$ and $
\norminf{q^*-q^{\hat\pi}}
\leq\epsilon$ 
with probability at least $1-\delta$, using at most
\begin{equation}
\label{eqn:exact_policy_sample_complexity}
C|S||A|
\crbk{
\frac{H}{(1-\gamma)^2\epsilon^2}
+\frac{l_\epsilon}{(1-\gamma)^3}
+\frac{1}{(1-\gamma)^5}}
\log\crbk{
\frac{8|S||A|l_\epsilon}{(1-\gamma)\delta}}
\end{equation}
transition samples.
\end{enumerate}
\end{corollary}

Thus, achieving the optimal Q-estimation and policy-learning sample complexity guarantees uses the same recentering batches. The difference only arises from the final inner loop: estimating $q^*$ to error tolerance $\epsilon$ requires the usual $(1-\gamma)^{-3}$ inner-loop iterations, whereas converting the final iterate estimation error into an $\epsilon$-optimal greedy policy certificate requires the larger $(1-\gamma)^{-5}$ number of iterations. Nevertheless, as a feature of the variance-reduced Q-learning, the inner loop does not need the number of iterations to scale with the target accuracy $\epsilon$, maintaining the minimax optimal leading term of $\widetilde O(H(1-\gamma)^{-2}\epsilon^{-2})$. 

\begin{remark}\label{rmk:knowledge_of_H}
Theorem~\ref{thm:pathwise_q_and_policy_error} and Corollary~\ref{cor:span_dependent_sample_complexity} remain valid if $H$ in the recentering batch sizes is replaced by any upper bound $\bar H\geq H$, with the resulting sample-complexity bounds stated in terms of $\bar H$. The universal choice $\bar H=(1-\gamma)\inv$ requires no additional MDP knowledge, but it does not retain the instance-dependent improvement when $H\ll(1-\gamma)\inv$.

Assume that $H\geq1$. The dependence on prior knowledge of $H$ can be removed through the following two-stage procedure. First, run the Q-estimation implementation in Corollary~\ref{cor:span_dependent_sample_complexity} with $H$ replaced by the universal upper bound $(1-\gamma)^{-1}$, failure probability $\eta/2$, and target accuracy $0.1$. This produces an estimator $\underline q$ satisfying $\norminf{\underline q-q^*}\leq0.1$
with probability at least $1-\eta/2$. By \eqref{eqn:exact_q_sample_complexity}, this pilot stage uses $\widetilde O\left(|S||A|(1-\gamma)^{-3}\right)$
transition samples.

Define $\bar H:=\spnorm{v(\underline q)}+\frac12.$
Then, with probability at least $1-\eta/2$, $H\leq\bar H\leq H+1\leq2H,$ where the final inequality uses $H\geq1$. In the second stage, run either the Q-estimation or policy-learning implementation of VRQL with $\bar H$ in place of $H$, failure probability $\eta/2$, and target accuracy $\epsilon$. Since $\bar H$ is a valid upper bound on $H$ with high probability, Corollary~\ref{cor:span_dependent_sample_complexity} applies with a sample size leading term of $\widetilde O\left(|S||A|\bar H(1-\gamma)^{-2}\epsilon^{-2}\right)$. 
Thus, by a union bound, the combined procedure returns an $\epsilon$-accurate Q-function or policy estimator with probability at least $1-\eta$, while preserving the optimal leading sample-complexity dependence. The pilot stage contributes only an additional $\epsilon$-independent sample size of order $\widetilde O(|S||A|(1-\gamma)^{-3})$.
\end{remark}

\begin{proof}[Proof of Corollary \ref{cor:span_dependent_sample_complexity}]
By the definition of $l_\epsilon$, we note that
\begin{equation}
\label{eqn:l_epsilon_accuracy_and_power}
\frac{2^{-l_\epsilon}}{1-\gamma}\leq\epsilon
\quad\text{and}\quad
4^{l_\epsilon}\leq \frac{4}{(1-\gamma)^2\epsilon^2}.
\end{equation}
The first inequality and Theorem~\ref{thm:pathwise_q_and_policy_error} imply the claimed $q$-error and policy-error guarantees.  

We count the total transition samples used. The recentering stage of epoch $l$ uses $n_l$ samples.  Since
\[
\sum_{l=1}^{l_\epsilon}4^l
=\frac{4^{l_\epsilon+1}-4}{3}
\leq \frac{4}{3}4^{l_\epsilon},
\]
\eqref{eqn:l_epsilon_accuracy_and_power} implies that, after increasing the universal constant $C$ to absorb the ceiling terms, 
\begin{align}
|S||A|\sum_{l=1}^{l_\epsilon}n_l
&\leq C|S||A|\crbk{
\frac{H}{(1-\gamma)^2\epsilon^2}
+\frac{l_\epsilon}{(1-\gamma)^2}}
\log\crbk{\frac{8|S||A|l_\epsilon}{\delta}}.
\label{eqn:recenter_sample_sum_bound}
\end{align}
For the standard implementation, the inner loops use $|S||A|l_\epsilon k^*$ samples. Noting that $(1-\gamma)^{-1} > 1$, we add this to the recentering samples to get \eqref{eqn:exact_q_sample_complexity}.

For the policy implementation, the total number of inner-loop samples is $|S||A|[(l_\epsilon-1)k^*+k^*_{\mathrm{policy}}].$
Combining this with \eqref{eqn:recenter_sample_sum_bound} gives \eqref{eqn:exact_policy_sample_complexity}.
\end{proof}

\section{Proof of Theorem \ref{thm:pathwise_q_and_policy_error}}\label{sec:proof:thm:pathwise_q_and_policy_error}

In this section, we prove Theorem~\ref{thm:pathwise_q_and_policy_error}. At a high level, the algorithm design can be viewed as a two-step approach. First, we establish a generic convergence bound for the inner-loop error, where the target fixed point in epoch $l$ is $\bar q_l$ defined in \eqref{eqn:def_bar_q_l}. Second, we choose the recentering sample size in each epoch so that, with high probability, the target fixed point $\bar q_l$ becomes geometrically closer to the true fixed point. 

Thus, the current iterate estimation error is decomposed as
\[
\norminf{\hat q_l - q^*}\leq \norm{\hat q_l - \bar q_l} + \norm{\bar q_l - q^*}. 
\]
The number of iterations per epoch is then chosen so that the inner-loop error $\norm{\hat q_l - \bar q_l}$ matches this geometric decay of the fixed-point error $\norm{\bar q_l - q^*}$. To obtain the last-iterate policy error bound, we refine the final-iterate estimate by increasing the length of the last epoch.

We note that the key insight for obtaining the optimal result is that the fixed-point error is controlled by the variance of the solution to
\[
\bar q_l(s,a)=\bar r_l(s,a)+\gamma P_{s,a}[v(\bar q_l)].
\]
where, conditional on the algorithm's progress before epoch $l$, the randomness arises only from the perturbed reward, rather than from the sampling variance of the transitions. Thus, after propagating the variance through the Bellman fixed point, it is bounded by $n_l^{-1}\norm{(I-\gamma P^{\hat\pi_l})\inv \sigma(\hat q_{l-1})}^2,$ where the factor $n_l^{-1}$ arises because $\bar r_l$ contains an $n_l$-sample recentering term anchored at $\hat q_{l-1}$.

Since the algorithm is designed so that $\hat q_l$ is geometrically close to $q^*$, the central quantity to control is therefore $\norm{(I-\gamma P^\pi)\inv\sigma^*}$ for near-optimal policies $\pi$. Such a bound is established in Proposition~\ref{prop:key_variance_bd}.

\subsection{Inner Loop Error Bound}
First, we establish the following bound on the inner-loop error relative to its fixed point, conditional on the estimator from the previous epoch and the recentering. 
\begin{lemma}
\label{lemma:quantitative_inner_loop_bound}
Fix a deterministic $q^\circ\in\R^{|S||A|}$ and conditioned on the recentering batch used to form $\widetilde\cT_l(q^\circ)$. Define
\begin{equation}
\widehat\cJ_{k+1}^\circ(q):=\widehat\cT_{k+1}(q)-\widehat\cT_{k+1}(q^\circ)+\widetilde\cT_l(q^\circ),
\quad
\cJ^\circ(q):=\cT(q)-\cT(q^\circ)+\widetilde\cT_l(q^\circ),
\end{equation}
and let $\bar q^\circ$ be the fixed point of $\cJ^\circ$. 

Initialize $q_1^\circ=q^\circ$ and run $q_{k+1}^\circ=(1-\lambda_k)q_k^\circ+\lambda_k\widehat\cJ_{k+1}^\circ(q_k^\circ)$ with $\lambda_k=1/(1+(1-\gamma)k)$. There is a universal constant $c\geq1$ such that, for every $\eta\in(0,1)$ and integer $K\geq1$ satisfying $K\geq \log\crbk{8|S||A|K/\eta}$, 
\begin{equation}
\label{eqn:conditional_inner_loop_bound}
\norminf{q_{K+1}^\circ-\bar q^\circ}
\leq c\norminf{q^\circ-\bar q^\circ}
\left[
\frac{\sqrt{\log\crbk{8|S||A|K/\eta}}}{1+(1-\gamma)K}
+\sqrt{\frac{\log\crbk{8|S||A|K/\eta}}{(1-\gamma)^3K}}
\right]
\end{equation}
with conditional probability at least $1-\eta$.
\end{lemma}

We note that this lemma is a more explicit version of Lemma~3 in \citet{wainwright2019}, and its proof follows the same argument. For completeness, we reproduce the proof in Appendix~\ref{section:proof:lemma:quantitative_inner_loop_bound}.

\subsection{Variance Coefficient Bound}

As explained before, in the next proposition we present the key structural estimate for the algorithm variance. 

\begin{proposition}
\label{prop:key_variance_bd}
For every stationary deterministic policy $\pi\in\Pi$,
\begin{equation}
\label{eqn:key_variance_bound}
\norminf{(I-\gamma P^\pi)\inv\sigma^*}^2
\leq \frac{3H}{2(1-\gamma)^2}
+\frac{H}{1-\gamma}\norminf{v^*-v^\pi}.
\end{equation}

\end{proposition}

We observe that an important feature of Proposition~\ref{prop:key_variance_bd} is its applicability to all policies $\pi$. In the proof of the algorithmic result, the policy to which it is applied can be the random $\bar q_l$-greedy policy $\bar\pi_l$ or the random $\hat q_l$-greedy policy $\hat\pi_l^*$. We note that if $\pi$ is optimal, then $v^* - v^\pi$ will vanish. 

\begin{proof}
Fix a stationary deterministic policy $\pi\in\Pi$, and define the centered value
\[
\bar v(s)
:=v^*(s)-\frac{\max_{s'\in S}v^*(s')+\min_{s'\in S}v^*(s')}{2},
\quad s\in S.
\]
Then $\norminf{\bar v}\leq H/2$. The centering gives
\begin{equation}
    \norminf{\bar v^2}\leq \frac{H^2}{4}
\quad\text{and}\quad
\norminf{\bar v+P_\pi[\bar v]}\leq H.
\label{eqn:tu_bar_v_bd}
\end{equation}
Let $\sigma_\pi^*(s):=\sigma^*(s,\pi(s))$. Since variance is unchanged by adding a constant, for every $s\in S$,
\[
\begin{aligned}
0\leq (\sigma_\pi^*)^2(s)
&=\gamma^2\set{P_\pi[\bar v^2](s)-\bar v(s)^2
+\bar v(s)^2-P_\pi[\bar v](s)^2}\\
&\leq\gamma^2\set{P_\pi[\bar v^2](s)-\bar v(s)^2
+\abs{\bar v(s)-P_\pi[\bar v](s)}
\abs{\bar v(s)+P_\pi[\bar v](s)}}\\
&\leq\gamma^2\set{P_\pi[\bar v^2](s)-\bar v(s)^2
+H\abs{\bar v(s)-P_\pi[\bar v](s)}}.
\end{aligned}
\]
Thus, applying the  entry-wise nonnegative matrix $(I-\gamma P_\pi)\inv=\sum_{k=0}^\infty\gamma^kP_\pi^k$, we have
\begin{equation}
\label{eqn:d1_d2_decomposition}
0\leq (I-\gamma P_\pi)\inv(\sigma_\pi^*)^2
\leq
\underbrace{\gamma^2(I-\gamma P_\pi)\inv\crbk{P_\pi[\bar v^2]-\bar v^2}}_{=:d_1}
+
\underbrace{\gamma^2H(I-\gamma P_\pi)\inv
\abs{\bar v-P_\pi[\bar v]}}_{=:d_2}.
\end{equation}
We analyze $d_1$ and $d_2$ separately.

For $d_1$, observe that
\[\begin{aligned}
\gamma(I-\gamma P_\pi)\inv
\crbk{P_\pi[\bar v^2]-\bar v^2}
&=\sum_{k=0}^\infty\gamma^{k+1}P_\pi^{k+1}[\bar v^2]
-\gamma\sum_{k=0}^\infty\gamma^kP_\pi^k[\bar v^2]\\
&=(1-\gamma)\sum_{k=0}^\infty\gamma^kP_\pi^k[\bar v^2]-\bar v^2.
\end{aligned}\]
The matrix $(1-\gamma)(I-\gamma P_\pi)\inv$ is stochastic. Hence both $(1-\gamma)(I-\gamma P_\pi)\inv\bar v^2$ and $\bar v^2$ take values in $[0,\norminf{\bar v^2}]$. Consequently, by \eqref{eqn:tu_bar_v_bd}, 
\begin{equation}
\label{eqn:tu_d1_bound}
\norminf{d_1}
\leq\gamma\norminf{\bar v^2}
\leq\frac{\gamma H^2}{4}.
\end{equation}

For $d_2$, first note that
\[
\bar v-P_\pi[\bar v]=v^*-P_\pi[v^*].
\]
Since $v^*-P_\pi[v^*]$ could contain negative entries, we attempt to restore non-negativity by considering the optimality equation. 
In particular, Bellman optimality implies that for any policy $\pi\in\Pi$, 
\[
v^*-r_\pi-\gamma P_\pi[v^*]\geq 0. 
\]
Therefore,
\[\begin{aligned}
\abs{\bar v-P_\pi[\bar v]}
&=\abs{r_\pi+\big(v^*-r_\pi-\gamma P_\pi[v^*]\big)
-(1-\gamma)P_\pi[v^*]}\\
&\leq 1-r_\pi+ (I-\gamma P_\pi)[v^*].
\end{aligned}\]
Here we used $r_\pi\in[0,1]$ and $(1-\gamma)P_\pi[v^*]\in[0,1]$. 

Applying the policy-evaluation identity $v^\pi=(I-\gamma P_\pi)\inv r_\pi$, we obtain
\[\begin{aligned}
0\leq (I-\gamma P_\pi)\inv\abs{\bar v-P_\pi[\bar v]}
&\leq \frac{1}{1-\gamma}+v^*-v^\pi.
\end{aligned}\]
Thus, it follows that
\begin{equation}
\label{eqn:tu_d2_bound}
\norminf{d_2}
\leq \frac{\gamma^2H}{1-\gamma}
+\gamma^2H\norminf{v^*-v^\pi}.
\end{equation}
Combining \eqref{eqn:d1_d2_decomposition}, \eqref{eqn:tu_d1_bound}, and \eqref{eqn:tu_d2_bound} yields the state-level estimate
\begin{equation}
\label{eqn:state_level_propagated_variance}
\norminf{(I-\gamma P_\pi)\inv(\sigma_\pi^*)^2}
\leq \frac{\gamma H^2}{4}
+\frac{\gamma^2H}{1-\gamma}
+\gamma^2H\norminf{v^*-v^\pi}.
\end{equation}

We next lift \eqref{eqn:state_level_propagated_variance} to the state-action space. After the initial state-action pair $(s,a)$, the process follows policy $\pi$. Hence the Neumann-series representation gives
\begin{equation}
\quad\sqbk{(I-\gamma P^\pi)\inv(\sigma^*)^2}(s,a)=(\sigma^*)^2(s,a)
+\gamma P_{s,a}\sqbk{(I-\gamma P_\pi)\inv(\sigma_\pi^*)^2}.
\label{eqn:lift_state_variance_to_q}
\end{equation}
Since $v^*$ has range length $H$, 
\[
(\sigma^*)^2(s,a)
=\gamma^2\var_{S'\sim P_{s,a}}(v^*(S'))
\leq\frac{\gamma^2H^2}{4}.
\]
Substituting this bound and \eqref{eqn:state_level_propagated_variance} into \eqref{eqn:lift_state_variance_to_q}, we obtain
\begin{equation}
\begin{aligned}
\norminf{(I-\gamma P^\pi)\inv(\sigma^*)^2}
&\leq \frac{\gamma^2H^2}{2}
+\frac{\gamma^3H}{1-\gamma}
+\gamma^3H\norminf{v^*-v^\pi}\\
&\leq \frac{3H}{2(1-\gamma)}
+H\norminf{v^*-v^\pi}.
\end{aligned}
    \label{eqn:propagated_variance_pre_simplification}
\end{equation}
The final inequality uses $H\leq(1-\gamma)\inv$.

It remains to pass from the propagated variance to the propagated standard deviation. For every stochastic matrix $P$ and every nonnegative vector $u$, Cauchy--Schwarz and Jensen's inequality give, coordinate-wise,
\[\begin{aligned}
\sqbk{(I-\gamma P)\inv u}(x)^2
&=\crbk{\sum_{k=0}^\infty\gamma^kP^k[u](x)}^2\\
&\leq\crbk{\sum_{k=0}^\infty\gamma^k}
\crbk{\sum_{k=0}^\infty\gamma^kP^k[u](x)^2}\\
&\leq\frac{1}{1-\gamma}
\sum_{k=0}^\infty\gamma^kP^k[u^2](x).
\end{aligned}\]
Therefore,
\begin{equation}
\label{eqn:resolvent_Cauchy_Schwarz}
\norminf{(I-\gamma P)\inv u}^2
\leq\frac{1}{1-\gamma}
\norminf{(I-\gamma P)\inv u^2}.
\end{equation}
Applying \eqref{eqn:resolvent_Cauchy_Schwarz} to \eqref{eqn:propagated_variance_pre_simplification} with $P=P^\pi$ and $u=\sigma^*$ proves \eqref{eqn:key_variance_bound}. 
\end{proof}

\subsection{Geometric Error Decay Among Epochs}

For $0\leq l\leq l^*$, define the targeted error levels after epoch $l$ by
\begin{equation}
\label{eqn:def_geometric_target_b_l}
b_l:=\frac{2^{-l}}{1-\gamma}.
\end{equation}
The initialization satisfies $\norminf{\hat q_0-q^*}\leq b_0$ deterministically.  We show that, conditional on $\norminf{\hat q_{l-1}-q^*}\leq b_{l-1}$, epoch $l$ reaches the target $b_l$ except on an event of conditional probability at most $\delta/l^*$. We note that this schedule $\set{b_l:l\ge 0}$ is different from that in \citet{wainwright2019}, where a warm start so that $\norminf{\hat q_0 - q^*}\leq1/\sqrt {1-\gamma }$ is required. 

\paragraph{Auxiliary bound: concentration of the recentering term.}
Fix an epoch $l\in\set{1,\ldots,l^*}$, condition on $\cG_{l-1}$, and consider sample paths $\omega$ so that 
\begin{equation}
\label{eqn:epoch_induction_hypothesis}
\norminf{\hat q_{l-1}-q^*}\leq b_{l-1}.
\end{equation}
Let $z_l$ denote the random reward perturbation for epoch $l$; i.e. 
\[
z_l:=\bar r_l-r
=\widetilde\cT_l(\hat q_{l-1})-\cT(\hat q_{l-1}),
\]
and set the Bellman error as
\[
e_{l-1}:=v(\hat q_{l-1})-v^*.
\]
The maximum operation is non-expansive in the sup norm, so \eqref{eqn:epoch_induction_hypothesis} implies $\norminf{e_{l-1}}\leq b_{l-1}$.  

For each $(s,a)$, we add and subtract $\cT(q^*)$ and get
\begin{equation}
z_l(s,a)=\frac{1}{n_l}\sum_{j=1}^{n_l}
\gamma\crbk{v^*(S'_{l,j}(s,a))-P_{s,a}[v^*]}+\frac{1}{n_l}\sum_{j=1}^{n_l}
\gamma\crbk{e_{l-1}(S'_{l,j}(s,a))-P_{s,a}[e_{l-1}]}.
\label{eqn:recenter_reward_decomposition}
\end{equation}
By the definition of $\sigma^* = \sigma(q^*)$, the summands in the first term of \eqref{eqn:recenter_reward_decomposition} have variance $(\sigma^*)^2(s,a)$ and absolute value at most $\gamma H$. Moreover, conditioned on $\cG_{l-1}$ and on the event that \eqref{eqn:epoch_induction_hypothesis} holds, the second term is an average of mean-centered random variables whose range is within an interval of length $2\gamma b_{l-1}$.  

Applying Bernstein's inequality for the first term, Hoeffding's inequality for the second term, and a union bound over the $|S||A|$ state-action pairs and the two terms, we have
\begin{equation}
\label{eqn:recenter_reward_vector_concentration}
|z_l|
\leq
(\sigma^*
+\gamma b_{l-1})\sqrt{\frac{2\log\crbk{8|S||A|l^*/\delta}}{n_l}}
+\frac{2\gamma H\log\crbk{8|S||A|l^*/\delta}}{3n_l}
\end{equation}
for all entries with conditional probability at least $1-\delta/(2l^*)$ on the event in \eqref{eqn:epoch_induction_hypothesis}.

For any policy $\pi\in\Pi$, $(I-\gamma P^\pi)\inv$ is entry-wise nonnegative. Thus, applying it to \eqref{eqn:recenter_reward_vector_concentration}, followed by Proposition~\ref{prop:key_variance_bd}, yields
\begin{equation}
\begin{aligned}
&\norminf{(I-\gamma P^\pi)\inv|z_l|}\\
&\leq
\sqrt{\frac{3H\log\crbk{8|S||A|l^*/\delta}}{n_l(1-\gamma)^2}} +\sqrt{\frac{2H\log\crbk{8|S||A|l^*/\delta}}{n_l(1-\gamma)}}
\sqrt{\norminf{v^*-v^\pi}}\\
&\quad +\frac{\gamma b_{l-1}}{1-\gamma}
\sqrt{\frac{2\log\crbk{8|S||A|l^*/\delta}}{n_l}}
+\frac{2\gamma H\log\crbk{8|S||A|l^*/\delta}}{3n_l(1-\gamma)}
\end{aligned}\label{eqn:uniform_policy_recenter_propagation}
\end{equation}
simultaneously for every policy $\pi\in\Pi$. 

For notational simplicity, we group the policy-independent error terms as
\begin{equation}
\label{eqn:def_eta_l_statistical_radius}
\zeta_l
:=\sqrt{\frac{3H\log\crbk{8|S||A|l^*/\delta}}{n_l(1-\gamma)^2}}
+\frac{\gamma b_{l-1}}{1-\gamma}
\sqrt{\frac{2\log\crbk{8|S||A|l^*/\delta}}{n_l}}+\frac{2\gamma H\log\crbk{8|S||A|l^*/\delta}}{3n_l(1-\gamma)}.
\end{equation}
This is the only epoch-specific scalar quantity introduced in the proof.  In this notation, \eqref{eqn:uniform_policy_recenter_propagation} reads
\begin{equation}
\label{eqn:uniform_policy_recenter_propagation_short}
\norminf{(I-\gamma P^\pi)\inv|z_l|}
\leq \zeta_l
+\sqrt{\frac{2H\log\crbk{8|S||A|l^*/\delta}}{n_l(1-\gamma)}}
\sqrt{\norminf{v^*-v^\pi}}.
\end{equation}

\paragraph{Error bound for the policy and fixed point under the perturbed reward.}
We first derive a bound on the policy error $\norminf{v^*-v^{\bar\pi_l}}$. Let $\bar v_l^\pi$ denote the value function of policy $\pi$ under the perturbed reward $\bar r_l=r+z_l$, and fix any original optimal policy $\pi^*\in\Pi^*$. Since $\bar\pi_l$ is optimal for the perturbed reward, $\bar v_l^{\bar\pi_l}\geq\bar v_l^{\pi^*}$. Consequently,
\[
\begin{aligned}
0\leq v^*-v^{\bar\pi_l}
&=v^{\pi^*}-v^{\bar\pi_l}\\
&=\crbk{v^{\pi^*}-\bar v_l^{\pi^*}}+\crbk{\bar v_l^{\pi^*}-\bar v_l^{\bar\pi_l}}+\crbk{\bar v_l^{\bar\pi_l}-v^{\bar\pi_l}}\\
&\leq\abs{v^{\pi^*}-\bar v_l^{\pi^*}}+\abs{\bar v_l^{\bar\pi_l}-v^{\bar\pi_l}}.
\end{aligned}
\]

For every policy $\pi$, define $z_{l,\pi}(s):=z_l(s,\pi(s))$. The policy-evaluation formulas under the rewards $r$ and $r+z_l$ give
\[
\bar v_l^\pi-v^\pi=(I-\gamma P_\pi)\inv z_{l,\pi}.
\]
Since $(I-\gamma P_\pi)\inv$ is entry-wise nonnegative,
\[
\abs{\bar v_l^\pi-v^\pi}\leq(I-\gamma P_\pi)\inv\abs{z_{l,\pi}}.
\]
Moreover, the definitions of $P_\pi$ and $P^\pi$ imply
\[
\sqbk{(I-\gamma P_\pi)\inv\abs{z_{l,\pi}}}(s)=\sqbk{(I-\gamma P^\pi)\inv\abs{z_l}}(s,\pi(s)).
\]
Therefore,
\begin{equation}
\label{eqn:fixed_policy_reward_perturbation}
\norminf{\bar v_l^\pi-v^\pi}\leq\norminf{(I-\gamma P^\pi)\inv\abs{z_l}}.
\end{equation}

On the recentering event \eqref{eqn:recenter_reward_vector_concentration}, which holds under the induction hypothesis \eqref{eqn:epoch_induction_hypothesis} with $\cG_{l-1}$-conditional probability at least $1-\delta/(2l^*)$, it follows from \eqref{eqn:fixed_policy_reward_perturbation} and \eqref{eqn:uniform_policy_recenter_propagation_short} that

\begin{equation}
 \begin{aligned}
\norminf{v^*-v^{\bar\pi_l}}
&\leq\norminf{(I-\gamma P^{\pi^*})\inv\abs{z_l}}+\norminf{(I-\gamma P^{\bar\pi_l})\inv\abs{z_l}}\\
&\leq2\zeta_l+\sqrt{\frac{2H\log\crbk{8|S||A|l^*/\delta}}{n_l(1-\gamma)}}\sqrt{\norminf{v^*-v^{\bar\pi_l}}},
\end{aligned}   \label{eqn:perturbed_policy_self_bound}
\end{equation}
where the square-root term vanishes for $\pi^*$ because $v^{\pi^*}=v^*$. Applying $xy\leq x^2/2+y^2/2$ to the final term in \eqref{eqn:perturbed_policy_self_bound} and rearranging gives
\begin{equation}
\label{eqn:perturbed_policy_loss_bound}
\norminf{v^*-v^{\bar\pi_l}}
\leq4\zeta_l+\frac{2H\log\crbk{8|S||A|l^*/\delta}}{n_l(1-\gamma)}.
\end{equation}

We next compare the two optimal action-value functions. Bellman optimality gives the entry-wise inequalities
\[
\bar q_l-q^*\leq(I-\gamma P^{\bar\pi_l})\inv z_l,
\quad
\bar q_l - q^*\ge (I-\gamma P^{\pi^*})\inv z_l.
\]
Therefore,
\begin{equation}
\label{eqn:bar_q_comparison}
\norminf{\bar q_l-q^*}
\leq\max\set{\norminf{(I-\gamma P^{\pi^*})\inv\abs{z_l}},\norminf{(I-\gamma P^{\bar\pi_l})\inv\abs{z_l}}}.
\end{equation}
The first quantity in the maximum is at most $\zeta_l$. For the second quantity, \eqref{eqn:uniform_policy_recenter_propagation_short} and \eqref{eqn:perturbed_policy_loss_bound} give
\[
\begin{aligned}
\norminf{(I-\gamma P^{\bar\pi_l})\inv\abs{z_l}}
&\leq\zeta_l+\sqrt{\frac{2H\log\crbk{8|S||A|l^*/\delta}}{n_l(1-\gamma)}}\sqrt{4\zeta_l+\frac{2H\log\crbk{8|S||A|l^*/\delta}}{n_l(1-\gamma)}}\\
&\leq2\zeta_l+\frac{4H\log\crbk{8|S||A|l^*/\delta}}{n_l(1-\gamma)},
\end{aligned}
\]
where the final inequality follows by applying $\sqrt{x+y}\leq\sqrt{x}+\sqrt{y}$ and $2xy\leq x^2+y^2$. Substituting this estimate into \eqref{eqn:bar_q_comparison} yields
\begin{equation}
\label{eqn:bar_q_statistical_bound}
\norminf{\bar q_l-q^*}
\leq2\zeta_l+\frac{4H\log\crbk{8|S||A|l^*/\delta}}{n_l(1-\gamma)}.
\end{equation}

We now use the prescribed batch size. Recalling that $b_{l-1}=2b_l$, the choice \eqref{eqn:n_l_span_dependent_choice} gives
\begin{equation}\label{eqn:tu_C_choice}
    \begin{aligned}
\frac{\zeta_l}{b_l}
&\leq\sqrt{\frac{3}{C}}+2\sqrt{\frac{2}{C}}+\frac{1}{3C},\\
\frac{H\log\crbk{8|S||A|l^*/\delta}}{n_l(1-\gamma)b_l}
&\leq\frac{2^{-l}}{C}
\leq\frac{1}{2C}.
\end{aligned}
\end{equation}
The right-hand sides depend only on $C$ and converge to zero as $C$ tends to infinity. We may therefore choose the universal constant $C$ sufficiently large that \eqref{eqn:perturbed_policy_loss_bound} and \eqref{eqn:bar_q_statistical_bound} imply
\begin{equation}
\label{eqn:bar_q_and_policy_epoch_targets}
\norminf{\bar q_l-q^*}\leq\frac{2}{5}b_l,
\quad
\norminf{v^*-v^{\bar\pi_l}}\leq\frac14b_l
\end{equation}
with $\cG_{l-1}$-conditional probability at least $1-\delta/(2l^*)$ on the event $\norminf{\hat q_{l-1}-q^*}\leq b_{l-1}.$

\paragraph{Inner-loop error from the perturbed fixed point.}
We now analyze the inner-loop error by viewing the recursion as Q-learning targeting the fixed point $\bar q_l$ of the perturbed Bellman operator $\cJ_l$.

Let $\cH_l$ be the sigma-field generated by $\cG_{l-1}$ and the recentering samples used in epoch $l$. Conditional on $\cH_l$, the operator $\cJ_l$, its fixed point $\bar q_l$, and the anchor $\hat q_{l-1}$ are deterministic. Apply Lemma~\ref{lemma:quantitative_inner_loop_bound} with $q^\circ=\hat q_{l-1}$ and failure probability $\eta=\delta/(2l^*)$.

For either $K=k^*$ or $K=k^*_{\mathrm{policy}}$ in \eqref{eqn:k_q_k_policy_choice}, there is a universal constant $c_0$ such that
\begin{equation}
\label{eqn:inner_log_adjustment}
\log\crbk{\frac{16|S||A|Kl^*}{\delta}}
\leq c_0\log\crbk{\frac{8|S||A|l^*}{(1-\gamma)\delta}}.
\end{equation}
Indeed, $\log\crbk{8|S||A|l^*/((1-\gamma)\delta)}\geq1$ and is at least $\log((1-\gamma)\inv)$, whereas both choices of $K$ are bounded by a universal constant times
\[
\frac{1}{(1-\gamma)^5}\log\crbk{\frac{8|S||A|l^*}{(1-\gamma)\delta}}.
\]
Hence $\log K$ is bounded by a universal constant times $\log\crbk{8|S||A|l^*/((1-\gamma)\delta)}$. After increasing $C$ if necessary, both choices of $K$ also satisfy
\[
K\geq\log\crbk{\frac{16|S||A|Kl^*}{\delta}},
\]
as required by Lemma~\ref{lemma:quantitative_inner_loop_bound}.

Lemma~\ref{lemma:quantitative_inner_loop_bound} therefore implies that, conditionally on $\cH_l$, with probability at least $1-\delta/(2l^*)$,
\begin{equation}\label{eqn:tu_q_inner_loop_error_to_target}
\norminf{q_{l,K+1}-\bar q_l}
\leq c\norminf{\hat q_{l-1}-\bar q_l}\left[\frac{\sqrt{\log\crbk{16|S||A|Kl^*/\delta}}}{1+(1-\gamma)K}+\sqrt{\frac{\log\crbk{16|S||A|Kl^*/\delta}}{(1-\gamma)^3K}}\right].
\end{equation}
For $K=k^*$, \eqref{eqn:inner_log_adjustment} implies that the two terms in brackets are bounded respectively by universal constants times $(1-\gamma)^2/C$ and $1/\sqrt C$. Increasing $C$ if necessary makes the entire contraction factor at most $1/4$. The same conclusion holds for $K=k^*_{\mathrm{policy}}$, since the longer inner loop only decreases both terms. Consequently,
\begin{equation}
\label{eqn:inner_loop_quarter_contraction}
\norminf{q_{l,K+1}-\bar q_l}
\leq\frac14\norminf{\hat q_{l-1}-\bar q_l}
\end{equation}
with $\cH_l$-conditional probability at least $1-\delta/(2l^*)$.

\paragraph{Union bound over the epochs.}
For the standard implementation, let $K_l=k^*$ for every epoch. For the policy implementation, let $K_l=k^*$ for $l<l^*$ and $K_{l^*}=k^*_{\mathrm{policy}}$. In either case, recall that $\hat q_l=q_{l,K_l+1}$.

On the intersection of \eqref{eqn:bar_q_and_policy_epoch_targets} and \eqref{eqn:inner_loop_quarter_contraction},
\begin{equation}
\begin{aligned}
\norminf{\hat q_l-q^*}
&\leq\norminf{\hat q_l-\bar q_l}+\norminf{\bar q_l-q^*}\\
&\leq\frac14\crbk{\norminf{\hat q_{l-1}-q^*}+\norminf{q^*-\bar q_l}}+\norminf{\bar q_l-q^*}\\
&\leq\frac14\crbk{2b_l+\frac25b_l}+\frac25b_l\\
&=b_l.
\end{aligned}\label{eqn:tu_good_event_vrql}
\end{equation}
We now make the probability accounting explicit because the fixed-point estimate and the inner-loop estimate are conditioned on different sigma-fields.

Define the following events on which the all the epochs satisfy the target error bounds $\set{b_l:1\leq l\leq l^*}$: 
\[
B_l:=\bigcap_{j=0}^l\set{\norminf{\hat q_j-q^*}\leq b_j}.
\]
The event $B_{l-1}$ is $\cG_{l-1}$-measurable. Conditional on $\cG_{l-1}$ and for sample paths $\omega\in B_{l-1}$, the first inequality in \eqref{eqn:bar_q_and_policy_epoch_targets} fails with probability at most $\delta/(2l^*)$. Therefore,
\[
\begin{aligned}
P\crbk{\underbrace{\set{\norminf{\bar q_l-q^*}>\frac25b_l}}_{F_{1,l}}\cap B_{l-1}}=E\left[\1_{ B_{l-1}}P\crbkcond{\norminf{\bar q_l-q^*}>\frac25b_l}{\cG_{l-1}}\right]\leq\frac{\delta}{2l^*}.
\end{aligned}
\]

Recall that $\cH_l$ is the $\sigma$-field generated by $\cG_{l-1}$ and the samples for computing the recentering operator $\widetilde \cT_l$. The event $\set{\norminf{\bar q_l-q^*}\leq2b_l/5}\cap B_{l-1}$ is $\cH_l$-measurable. Conditional on $\cH_l$, the inner-loop estimate \eqref{eqn:inner_loop_quarter_contraction} fails with probability at most $\delta/(2l^*)$. Hence,
\[
\begin{aligned}
&P\crbk{\underbrace{\set{\norminf{\hat q_l-\bar q_l}>\frac14\norminf{\hat q_{l-1}-\bar q_l}}}_{F_{2,l}}\cap\set{\norminf{\bar q_l-q^*}\leq\frac25b_l}\cap B_{l-1}}\\
&\quad=E\left[\1\set{\set{\norminf{\bar q_l-q^*}\leq2b_l/5}\cap B_{l-1}}P\crbkcond{\norminf{\hat q_l-\bar q_l}>\frac14\norminf{\hat q_{l-1}-\bar q_l}}{\cH_l}\right]\\
&\quad\leq\frac{\delta}{2l^*}.
\end{aligned}
\]

Whenever $ B_{l-1}$ holds and neither of these two failures $F_{1,l}$ and $F_{2,l}$ occurs, \eqref{eqn:tu_good_event_vrql} gives $\norminf{\hat q_l-q^*}\leq b_l$. Thus,
\[
P\crbk{ B_l^c\cap B_{l-1}}\leq P\crbk{ (F_{1,l}\cup F_{2,l})\cap B_{l-1}}\leq\frac{\delta}{l^*},
\]
and hence
\[
P( B_l^c)\leq P( B_{l-1}^c)+\frac{\delta}{l^*}.
\]
Since $P( B_0)=1$, induction gives $P( B_{l^*}^c)\le \delta.$
This proves part~\textnormal{(i)} of Theorem~\ref{thm:pathwise_q_and_policy_error}. The same argument proves the pathwise $q$-error bound in part~\textnormal{(ii)}, because the longer final inner loop also satisfies \eqref{eqn:inner_loop_quarter_contraction}.

\subsection{Last-Iterate Greedy Policy Error}
To finish the proof of Theorem \ref{thm:pathwise_q_and_policy_error}, it remains to analyze the greedy policy produced by the longer final inner loop. Thus, we specialize to $l=l^*$ and write $\hat\pi:=\hat\pi_{l^*}^*,$ $\bar v_l:=v(\bar q_l).$
Because $\hat\pi$ is greedy with respect to $\hat q_l$, for every state $s$,
\[
\begin{aligned}
0\leq\bar v_l(s)-\bar q_l(s,\hat\pi(s))
&=\bar q_l(s,\bar\pi_l(s))-\bar q_l(s,\hat\pi(s))\\
&\leq\abs{\bar q_l(s,\bar\pi_l(s))-\hat q_l(s,\bar\pi_l(s))}+\abs{\hat q_l(s,\hat\pi(s))-\bar q_l(s,\hat\pi(s))}\\
&\leq2\norminf{\hat q_l-\bar q_l}.
\end{aligned}
\]

We next use this one-step greedy error to control the value loss in the perturbed MDP. The policy evaluation equations for $\bar v_l^{\hat\pi}$ and $\bar q_l(\cd,\hat\pi(\cd))$ are
\[
\bar v_l^{\hat\pi}=\bar r_{l,\hat\pi}+\gamma P_{\hat\pi}\bar v_l^{\hat\pi},
\quad
\bar q_l(\cd,\hat\pi(\cd))=\bar r_{l,\hat\pi}+\gamma P_{\hat\pi}\bar v_l.
\]
Subtracting these identities gives
\[
(I-\gamma P_{\hat\pi})\crbk{\bar v_l-\bar v_l^{\hat\pi}}=\bar v_l-\bar q_l(\cd,\hat\pi(\cd)).
\]
The vector on the right-hand side is entry-wise nonnegative and bounded by $2\norminf{\hat q_l-\bar q_l}$. Since $(I-\gamma P_{\hat\pi})\inv$ is entry-wise nonnegative,
\begin{equation}
\label{eqn:approx_greedy_perturbed_policy_loss}
\norminf{\bar v_l-\bar v_l^{\hat\pi}}
\leq\frac{2\norminf{\hat q_l-\bar q_l}}{1-\gamma}.
\end{equation}

We now compare the original and perturbed rewards. Since $\pi^*$ is optimal for the original MDP and $\bar v_l$ is optimal for the perturbed MDP,
\[
\begin{aligned}
0\leq v^*-v^{\hat\pi}
&=\crbk{v^{\pi^*}-\bar v_l^{\pi^*}}+\crbk{\bar v_l^{\pi^*}-\bar v_l}+\crbk{\bar v_l-\bar v_l^{\hat\pi}}+\crbk{\bar v_l^{\hat\pi}-v^{\hat\pi}}\\
&\leq\abs{v^{\pi^*}-\bar v_l^{\pi^*}}+\crbk{\bar v_l-\bar v_l^{\hat\pi}}+\abs{\bar v_l^{\hat\pi}-v^{\hat\pi}}.
\end{aligned}
\]
Taking the sup norm, applying \eqref{eqn:fixed_policy_reward_perturbation} to the first and third terms, and then using \eqref{eqn:approx_greedy_perturbed_policy_loss}, gives
\[
\norminf{v^*-v^{\hat\pi}}
\leq\norminf{(I-\gamma P^{\pi^*})\inv\abs{z_l}}+\norminf{(I-\gamma P^{\hat\pi})\inv\abs{z_l}}+\frac{2\norminf{\hat q_l-\bar q_l}}{1-\gamma}.
\]

We apply the same analysis as in \eqref{eqn:perturbed_policy_self_bound} to the random policy $\hat\pi$ to get
\[
\norminf{v^*-v^{\hat\pi}}
\leq2\zeta_l+\sqrt{\frac{2H\log\crbk{8|S||A|l^*/\delta}}{n_l(1-\gamma)}}\sqrt{\norminf{v^*-v^{\hat\pi}}}+\frac{2\norminf{\hat q_l-\bar q_l}}{1-\gamma}.
\]
Applying $xy\leq x^2/2+y^2/2$ to the square-root term and rearranging yields
\begin{equation}
\label{eqn:actual_greedy_policy_bound}
\norminf{v^*-v^{\hat\pi}}
\leq4\zeta_l+\frac{2H\log\crbk{8|S||A|l^*/\delta}}{n_l(1-\gamma)}+\frac{4\norminf{\hat q_l-\bar q_l}}{1-\gamma}.
\end{equation}
We note that the term in \eqref{eqn:actual_greedy_policy_bound} has an additional factor of $1/(1-\gamma)$. Recall from \eqref{eqn:tu_q_inner_loop_error_to_target} that the inner loop conditional error can be controlled as $\norm{\hat q_{l} - \bar  q_l}\lesssim  \norm{\hat q_{l-1} - \bar  q_l}\widetilde O ({ K^{-1/2}})$. Thus, to control this $1/(1-\gamma)$ inflated term and the targeted geometric convergence, we need an additional $1/(1-\gamma)^2$ multiplicative factor on top of $k^*$, which gives $k^*_\mathrm{policy}$.

Specifically, for $K=k^*_{\mathrm{policy}}\ge k^*$, Lemma~\ref{lemma:quantitative_inner_loop_bound} and the fixed-point estimate in \eqref{eqn:bar_q_and_policy_epoch_targets} imply that 
\[
\norminf{\hat q_{l-1}-\bar q_l}
\leq\norminf{\hat q_{l-1}-q^*}+\norminf{q^*-\bar q_l}
\leq b_{l-1}+\frac25b_l
=\frac{12}{5}b_l
\]
conditioned on $\cG_{l-1}$ with probability at least $1-\delta/(2l^*)$ on the event $\norminf{\hat q_{l-1}-q^*}\leq b_{l-1}$. 

Recall from \eqref{eqn:tu_q_inner_loop_error_to_target} that
\[
\norminf{\hat q_{l}-\bar q_l}
\leq c\norminf{\hat q_{l-1}-\bar q_l}\left[\frac{\sqrt{\log\crbk{16|S||A|Kl^*/\delta}}}{1+(1-\gamma)K}+\sqrt{\frac{\log\crbk{16|S||A|Kl^*/\delta}}{(1-\gamma)^3K}}\right]\]
with $\cH_l$-conditional probability at least $1-\delta/(2l^*)$. By the remark thereafter, the two terms in the square bracket are bounded respectively by universal constants times $(1-\gamma)^4/C$ and $(1-\gamma)/\sqrt C$. Increasing $C$ if necessary can make the entire factor at most $(1-\gamma)/24$. Consequently,
\begin{equation}
\label{eqn:final_inner_alpha_accuracy}
\norminf{\hat q_l-\bar q_l}
\leq \frac{1-\gamma}{24}\norminf{\hat q_{l-1}-\bar q_l} \leq\frac{1-\gamma}{10}b_l.
\end{equation}

Finally, the same choice of $C$ used in \eqref{eqn:tu_C_choice} ensures that the first two terms on the right-hand side of \eqref{eqn:actual_greedy_policy_bound} are at most $b_l/4$. Combining this estimate with \eqref{eqn:final_inner_alpha_accuracy} gives that on the intersection of these high-probability events
\[
\norminf{v^*-v^{\hat\pi}}
\leq\frac14b_l+\frac{4}{1-\gamma}\cd\frac{1-\gamma }{10}b_l
\leq b_l.
\]
By the tower property and a union bound, this happens with $\cG_{l-1}$-conditional probability at least $1-\delta/l^*$ on $B_{l-1}$.
Moreover,
\[
q^*-q^{\hat \pi }=\gamma P[v^*-v^{\hat \pi}].
\]
Since $v^*\geq v^{\hat \pi}$ entry-wise,
\[
\norminf{q^*-q^{\hat \pi}}\leq\gamma\norminf{v^*-v^{\hat\pi}}\leq b_l.
\]
Replacing the last epoch failure event in the earlier across-epoch union bound argument with the policy version gives \eqref{eqn:pathwise_policy_error}. This completes the proof of Theorem \ref{thm:pathwise_q_and_policy_error}.

\section{Span-Error Sample Complexity of VRQL}
\label{section:span_seminorm_uniform_mixing}

The preceding results measure the estimation error in the ordinary sup norm and use the span of the optimal value function to control the stochastic Bellman error. We next observe that, for value-based RL, it is sufficient to learn the optimal Q-function up to a constant shift in order to recover an optimal policy. This motivates the following research questions:
\begin{enumerate}
    \item Can we leverage the stability of the MDP to learn the optimal Q-function up to a constant shift at a faster rate that is independent of $1/(1-\gamma)$?
    \item If the Q-function can be learned up to a constant shift at a faster rate, does this necessarily imply faster policy learning?
\end{enumerate}

We investigate these questions by measuring the estimation error in the span seminorm, which exactly captures error up to a constant shift. In this section, we show that under the stronger stability condition that every policy induces a uniformly ergodic Markov chain, VRQL achieves $\epsilon$-error in the span seminorm with sample complexity $\widetilde O(|S||A|\tmix^3\epsilon^{-2})$, which is optimal. However, in the next section, we show that, even when every policy induces mixing, this accurate estimation in the span seminorm does not in general imply that an $\epsilon$-optimal policy can be learned with sample complexity independent of $1/(1-\gamma)$.

Specifically, we consider the following assumption. 

\begin{assumption}
\label{assump:uniform_mixing}
All policies $\pi\in\Pi$ induce a mixing Markov chain. Let
\begin{equation}
\label{eqn:def_uniform_policy_mixing_time}
\tmix:=\max_{\pi\in\Pi}\tmix(P_\pi) < \infty.
\end{equation}
\end{assumption}

We show that the same VRQL framework can be modified to achieve the convergence result in the next theorem, whose proof is deferred to the appendix. 

\begin{theorem}[Pathwise span error under uniform mixing]
\label{thm:pathwise_span_error_uniform_mixing}
Under Assumption \ref{assump:uniform_mixing}, there exists a universal constant $C\geq1$ such that the following statement holds. Fix $l^*\geq1$ and $\delta\in(0,1)$. For $1\leq l\leq l^*$, choose
\begin{equation}
\label{eqn:n_l_span_uniform_mixing_choice}
n_l
:=\ceil{C\crbk{4^l\tmix+\tmix^2}\log\crbk{\frac{8|S||A|l^*}{\delta}}},
\end{equation}
and set
\begin{equation}
\label{eqn:k_span_uniform_mixing_choice}
k_l=k^*
:=\ceil{\frac{C}{(1-\gamma)^3}\log\crbk{\frac{8|S||A|l^*}{(1-\gamma)\delta}}}.
\end{equation}
Run Algorithm~\ref{alg:vr_q-learning} from $\hat q_0\equiv0$ with these parameters. Then, with probability at least $1-\delta$,
\begin{equation}
\label{eqn:pathwise_span_q_error}
\spnorm{\hat q_l-q^*}
\leq 8\cdot 2^{-l}\tmix,
\quad 0\leq l\leq l^*.
\end{equation}
\end{theorem}

As in Theorem~\ref{thm:pathwise_q_and_policy_error}, Theorem~\ref{thm:pathwise_span_error_uniform_mixing} also does not require a warm start from a separate initialization procedure. In particular, the $\widetilde O(\tmix^2)$ term in \eqref{eqn:n_l_span_uniform_mixing_choice} serves as an implicit warm-up phase.

\begin{corollary}[Sample complexity in the span seminorm]
\label{cor:span_uniform_mixing_sample_complexity}
Let $\epsilon\in(0,\sqrt{\tmix})$ and $\delta\in(0,1)$, and define
\begin{equation}
\label{eqn:def_l_epsilon_span_uniform_mixing}
l_\epsilon
:=\ceil{\log_2\crbk{\frac{8\tmix}{\epsilon}}}.
\end{equation}
Then, there exists a universal constant $C\geq1$ such that, when the parameters in Theorem~\ref{thm:pathwise_span_error_uniform_mixing} are chosen with $l^*=l_\epsilon$, the resulting estimator satisfies $\spnorm{\hat q_{l_\epsilon}-q^*}\leq\epsilon$
with probability at least $1-\delta$. The total number of transition samples is at most
\begin{equation}
\label{eqn:exact_span_q_sample_complexity}
C|S||A|\left[\frac{\tmix^3}{\epsilon^2}+\tmix^2\log(2\tmix)+\frac{l_\epsilon}{(1-\gamma)^3}\right]\log\crbk{\frac{8|S||A|l_\epsilon}{(1-\gamma)\delta}}.
\end{equation}
\end{corollary}

Since $\epsilon^2<\tmix$, the leading dependence in \eqref{eqn:exact_span_q_sample_complexity}, up to logarithmic factors, is
\[
\widetilde O\left(|S||A|\left[\frac{\tmix^3}{\epsilon^2}+\frac{1}{(1-\gamma)^3}\right]\right).
\]

\begin{remark}
As in Remark~\ref{rmk:knowledge_of_H}, Theorem~\ref{thm:pathwise_span_error_uniform_mixing} and Corollary~\ref{cor:span_uniform_mixing_sample_complexity} require knowledge of $\tmix$ to attain the stated rate. More generally, any upper bound
$\bar t_{\mathrm{mix}}
\geq
\max_{\pi\in\Pi}\tmix(P_\pi)$
yields a leading sample complexity of
$\widetilde O\left(|S||A|\bar t_{\mathrm{mix}}^3\epsilon^{-2}\right).$ Unfortunately, unlike the earlier case, the same two-stage calibration procedure cannot be applied. The algorithm needs an upper bound on the mixing times of all policies and therefore cannot be inferred from an accurate estimate of the optimal Q-function alone. Developing a parameter-free model-free procedure that retains the optimal $\tmix^3\epsilon^{-2}$ dependence is an important direction for future work and is beyond the scope of this paper.
\end{remark}

\begin{proof}[Proof of Corollary \ref{cor:span_uniform_mixing_sample_complexity}]
By the definition of $l_\epsilon$, we have
\begin{equation}
\label{eqn:l_epsilon_span_accuracy_and_power}
8\cdot2^{-l_\epsilon}\tmix\leq\epsilon,
\qquad
4^{l_\epsilon}\leq\frac{256\tmix^2}{\epsilon^2}.
\end{equation}
The first inequality and Theorem~\ref{thm:pathwise_span_error_uniform_mixing} imply the span-error guarantee.

We count the total number of transition samples used. The recentering stage of epoch $l$ uses $n_l$ one-sample Bellman operators. Splitting the sum according to whether $4^l\leq\tmix$ gives
\[
\sum_{l=1}^{l_\epsilon}\crbk{4^l\tmix+\tmix^2}
\leq C\left[4^{l_\epsilon}\tmix+\tmix^2\log(2\tmix)\right].
\]
Therefore, \eqref{eqn:n_l_span_uniform_mixing_choice} and \eqref{eqn:l_epsilon_span_accuracy_and_power} imply, after increasing the universal constant $C$ to absorb the ceiling terms, that
\begin{equation}
\label{eqn:span_recenter_sample_sum_bound}
|S||A|\sum_{l=1}^{l_\epsilon}n_l
\leq
C|S||A|
\left[
\frac{\tmix^3}{\epsilon^2}
+\tmix^2\log(2\tmix)
\right]
\log\crbk{\frac{8|S||A|l_\epsilon}{\delta}}.
\end{equation}

The inner loops use $|S||A|l_\epsilon k^*$ transition samples. Noting that $(1-\gamma)\inv\geq1$, we add this to \eqref{eqn:span_recenter_sample_sum_bound} and obtain \eqref{eqn:exact_span_q_sample_complexity}.
\end{proof}

\section{Minimax Lower Bounds}
\label{section:minimax_lower_bounds}

In this section, we provide matching lower bounds for the leading sample complexity dependencies of the three upper bounds established earlier in the paper, with $|S|,|A|=O(1)$. A single MDP family is used to prove all three results. In particular, the family shows that the leading span-seminorm sample complexity in Corollary~\ref{cor:span_uniform_mixing_sample_complexity} is minimax optimal, while sup-norm estimation and policy learning retain the additional factor $(1-\gamma)^{-2}$. In particular, this implies that while the Q-function can be learned up to a constant shift at a faster rate, one cannot achieve a faster policy error convergence rate in a minimax sense. 

For the remainder of this section, fix $S=\set{0,1}$ and $A=\set{a_1,a_2}$. For an MDP instance $\cM=(S,A,r,P,\gamma)$, let $q^*(\cM),v^*(\cM)$ be the optimal Q and value function of $\cM$. Moreover, for stationary deterministic $\pi\in\Pi$, let $v^\pi(\cM)$ denote the solution to the evaluation equation $v^\pi(\cM) = r_\pi +\gamma P_\pi v^\pi(\cM)$. 

We consider a measurable space $(\Omega,\cF)$ that supports a sequence of next-state samples for every state-action pair $\set{S'_n(s,a):s\in S,a\in A,n=1,2,\ds}$. Let $P^\cM$ be a probability measure on $(\Omega,\cF)$ such that, under $P^\cM$, the random variables $S'_n(s,a)$ are independent and satisfy $S'_n(s,a)\sim P_{s,a}$ for every $(s,a)\in S\times A$ and $n=1,2,\ds$. Define the data process
\[
\mathbf S_n:=\set{S_i'(s,a):1\leq i\leq n,\ s\in S,\ a\in A}
\]
and let $\cF_n:=\sigma(\bd S_n)$. 
We use $\cL_\cM(\mathbf S_n)$ to denote the law of $\mathbf S_n$ under $P^\cM$. With sample size $n$, a Q-function estimator $\tilde q_n$ is an $\cF_n$-measurable map from $\Omega$ to $\R^{S\times A}$, and a policy estimator $\tilde\pi_n$ is an $\cF_n$-measurable map from $\Omega$ to the class $\Pi$ of stationary deterministic Markov policies. Then, both $\tilde q_n$ and $\tilde\pi_n$ are identified as some measurable functions of $\bd S_n$. 

For an integer $\tmix\geq1$, let $\cC(\gamma,\tmix)$ denote the class of $\gamma$-discounted MDP instances on these state and action spaces with rewards in $[0,1]$ such that
\[
\max_{\pi\in\Pi}\tmix(P_\pi)\leq\tmix.
\]

\begin{theorem}[Minimax lower bounds under uniform mixing]
\label{thm:lower_bounds}
There is a universal constant $c_{\mathrm{lb}}>0$ such that the following statements hold for every $\gamma\in[0.9,1)$ and every integer $10\leq\tmix\leq(1-\gamma)^{-1}$. 

\begin{enumerate}[label=\textnormal{(\roman*)}]
\item For every $0<\epsilon\leq1/[112(1-\gamma)]$ and every
\[
n\leq c_{\mathrm{lb}}\frac{\tmix}{(1-\gamma)^2\epsilon^2},
\]
we have
\[
\inf_{\tilde q_n}\sup_{\cM\in\cC(\gamma,\tmix)}
P^\cM\crbk{\norminf{\tilde q_n-q^*(\cM)}>\epsilon}
\geq\frac{1}{4}
\]
and
\[
\inf_{\tilde\pi_n}\sup_{\cM\in\cC(\gamma,\tmix)}
P^\cM\crbk{\norminf{v^*(\cM)-v^{\tilde\pi_n}(\cM)}>\epsilon}
\geq\frac{1}{4}.
\]

\item For every $0<\epsilon\leq\tmix/96$ and every
\[
n\leq c_{\mathrm{lb}}\frac{\tmix^3}{\epsilon^2},
\]
we have
\[
\inf_{\tilde q_n}\sup_{\cM\in\cC(\gamma,\tmix)}
P^\cM\crbk{\spnorm{\tilde q_n-q^*(\cM)}>\epsilon}
\geq\frac{1}{4}.
\]

\end{enumerate}
\end{theorem}

Consequently, the minimax transition-sample complexities of sup-norm $q^*$ estimation and policy learning are $\Omega\crbk{\tmix(1-\gamma)^{-2}\epsilon^{-2}},$
whereas the minimax transition-sample complexity of span-seminorm $q^*$ estimation is $\Omega\crbk{\tmix^3\epsilon^{-2}}.$ Since $H\lesssim \tmix$ for any MDP, Theorem~\ref{thm:lower_bounds} together with Corollary~\ref{cor:span_dependent_sample_complexity} implies that the $H(1-\gamma)^{-2}\epsilon^{-2}$ dependence is optimal for both Q-estimation and policy learning. The lower bound also matches the upper bound in its $\tmix$ dependence for span-error Q-estimation.

The following parametric family is used for all three parts of Theorem~\ref{thm:lower_bounds}; only the size of the perturbation is calibrated differently for the two estimation metrics.

\begin{definition}[Hard MDP family]
\label{def:joint_hard_mdp_family}
Fix $\gamma\in[0.9,1)$ and an integer $10\leq\tmix\leq(1-\gamma)^{-1}$, and set $p:=1/\tmix$. For $\dagger\in\set{+,-}$ and $0<\Delta\leq p/2$, let $\cM_{\dagger,\Delta}$ be the MDP with state space $S=\set{0,1}$, action space $A=\set{a_1,a_2}$, and rewards
\[
r(0,a)=0,
\quad
r(1,a)=1,
\quad
a\in A.
\]
At state $1$, both actions have the same transition law:
\[
P_{\dagger,\Delta}(0\mid1,a)=p,
\quad
P_{\dagger,\Delta}(1\mid1,a)=1-p,
\quad
a\in A.
\]
At state $0$, action $a_2$ has transition probability $p$ to state $1$ in both instances, whereas action $a_1$ has transition probability $p+\Delta$ in $\cM_{+,\Delta}$ and $p-\Delta$ in $\cM_{-,\Delta}$:
\[
\begin{aligned}
P_{+,\Delta}(1\mid0,a_1)&=p+\Delta,
&\quad
P_{-,\Delta}(1\mid0,a_1)&=p-\Delta,\\
P_{+,\Delta}(1\mid0,a_2)&=p,
&
P_{-,\Delta}(1\mid0,a_2)&=p.
\end{aligned}
\]
All remaining transition probabilities are defined by complementation.
\end{definition}

\begin{figure}[htb]
\centering
\begin{tikzpicture}[>=latex,thick]
\node[draw,circle,minimum size=1.1cm] (zero) at (0,0) {$0$};
\node[draw,circle,minimum size=1.1cm] (one) at (6,0) {$1$};
\node at (0,-0.9) {$r(0,\cd)=0$};
\node at (6,-0.9) {$r(1,\cd)=1$};

\draw[->] (zero) to[bend left=18]
node[above,align=center] {$a_1:p\pm\Delta$\\[-1mm]$a_2:p$} (one);
\draw[->] (one) to[bend left=18]
node[below,align=center] {$a_1,a_2:p$} (zero);

\draw[->] (zero) to[out=145,in=215,looseness=6]
node[left,align=center] {$a_1:1-p\mp\Delta$\\[-1mm]$a_2:1-p$} (zero);
\draw[->] (one) to[out=35,in=-35,looseness=6]
node[right,align=center] {$a_1,a_2:1-p$} (one);
\end{tikzpicture}
\caption{The two-state hard MDP family. The upper signs in $(\pm,\mp)$ correspond to $\cM_{+,\Delta}$ and the lower signs correspond to $\cM_{-,\Delta}$.}
\label{fig:joint_hard_mdp_family}
\end{figure}

\begin{proof}[Proof of Theorem~\ref{thm:lower_bounds}]
We divide the proof into four steps. First, we verify that all MDP instances in the hard family have mixing time at most $\tmix$. Second, we bound the statistical distance between the two generative-model experiments and establish the two-point testing inequality used below. Third, we calculate the exact policy values and verify the deterministic Q-function and policy-value gaps required by the testing argument. Finally, we calibrate the perturbation $\Delta$ and apply the two-point inequality to prove the three lower bounds.

\paragraph{Step 1: uniform mixing of every policy.}
Every stationary deterministic policy induces a two-state transition kernel of the form
\[
K_x
=
\bmx{
1-x & x\\
p & 1-p
},
\qquad
x\in\set{p-\Delta,p,p+\Delta}.
\]
Its stationary distribution is $\eta=(p/(p+x),x/(p+x))$ and its second eigenvalue is $1-p-x$. A direct two-state calculation gives
\[
\normTV{K_x^k(0,\cd)-\eta}
=
\frac{x}{p+x}(1-p-x)^k,
\qquad
\normTV{K_x^k(1,\cd)-\eta}
=
\frac{p}{p+x}(1-p-x)^k.
\]
In particular,
\[
\sup_{s\in S}\normTV{K_x^k(s,\cd)-\eta}
\leq
(1-p-x)^k
\]
for every $k\geq0$. Moreover, since $p=1/\tmix$ and $x\geq p/2$, we have
\[
(1-p-x)^{\tmix}
\leq
\crbk{1-\frac{3}{2\tmix}}^{\tmix}
\leq
e^{-3/2}
<
\frac14;
\]
i.e., $\tmix(K_x)\leq\tmix$ for every policy-induced kernel $K_x$ for both instances $\dagger\in\set{+,-}$ when $0<\Delta\leq p/2$. Therefore, $\cM_{+,\Delta},\cM_{-,\Delta}
\in
\cC(\gamma,\tmix).$

\paragraph{Step 2: statistical divergence and the two-point testing bound.}
In this step, we quantify the statistical distance between the laws of the data $\mathbf S_n$ under the two hard instances and use this distance to obtain a general lower bound on the probability of estimation error.

The two laws $\cL_{\cM_{+,\Delta}}(\mathbf S_n)$ and $\cL_{\cM_{-,\Delta}}(\mathbf S_n)$, viewed as distributions of $n|S||A|$ independent random variables taking values in $S=\set{0,1}$, differ only in the coordinates $\set{S_i'(0,a_1):1\leq i\leq n}$. Indeed, all other entries of the transition kernel are identical under the two instances, while $S_i'(0,a_1) \sim \bnl(p\pm\Delta)$
under $\cM_{\pm,\Delta}$. Consequently, by the factorization property of the KL divergence,
\[
\begin{aligned}
\KL\crbk{\cL_{\cM_{+,\Delta}}(\mathbf S_n)\,\|\,\cL_{\cM_{-,\Delta}}(\mathbf S_n)}
&=
n\,\KL\crbk{\bnl(p+\Delta)\,\|\,\bnl(p-\Delta)}\\
&\leq
n\frac{(2\Delta)^2}{(p-\Delta)(1-p+\Delta)}\\
&\leq
\frac{16n\Delta^2}{p}.
\end{aligned}
\]
Here the first inequality is the elementary Bernoulli bound
\[
\KL\crbk{\bnl(u)\,\|\,\bnl(v)}
\leq
\frac{(u-v)^2}{v(1-v)},
\]
and the last inequality uses $\Delta\leq p/2$ and $p\leq1/10$. Therefore, Pinsker's inequality gives
\[
\normTV{\cL_{\cM_{+,\Delta}}(\mathbf S_n)-\cL_{\cM_{-,\Delta}}(\mathbf S_n)}
\leq
\sqrt{\frac12
\KL\crbk{\cL_{\cM_{+,\Delta}}(\mathbf S_n)\,\|\,\cL_{\cM_{-,\Delta}}(\mathbf S_n)}}\leq
\sqrt{\frac{8n\Delta^2}{p}}.
\]

We will use Le Cam's two-point testing argument. Let $d$ be a pseudometric, which will later be either the sup norm or the span seminorm, and suppose that two targets $\theta_+$ and $\theta_-$ satisfy $d(\theta_+,\theta_-)\geq4\epsilon.$
For any $\cF_n$-measurable estimator $\tilde\theta_n$, define
\[
A_+:=\set{d(\tilde\theta_n,\theta_+)\leq\epsilon},
\qquad
A_-:=\set{d(\tilde\theta_n,\theta_-)\leq\epsilon}.
\]
The triangle inequality and the separation of the targets imply that $A_+\cap A_-=\varnothing$. Thus,
\begin{equation}
\begin{aligned}
\max_{\dagger\in\set{+,-}}
P^{\cM_{\dagger,\Delta}}
\crbk{d(\tilde\theta_n,\theta_\dagger)>\epsilon}
&\geq
\frac12\left[
P^{\cM_{+,\Delta}}(A_+^c)
+
P^{\cM_{-,\Delta}}(A_-^c)
\right]\\
&\stackrel{\textnormal{(i)}}{\geq}
\frac12\left[
P^{\cM_{+,\Delta}}(A_+^c)
+
P^{\cM_{-,\Delta}}(A_+)
\right]\\
&=
\frac12\left[
1-
\crbk{
P^{\cM_{+,\Delta}}(A_+)
-
P^{\cM_{-,\Delta}}(A_+)
}
\right]\\
&\geq
\frac12\left[
1-
\normTV{\cL_{\cM_{+,\Delta}}(\mathbf S_n)-\cL_{\cM_{-,,\Delta}}(\mathbf S_n)}
\right]\\
&\geq
\frac12-\sqrt{\frac{2n\Delta^2}{p}},
\end{aligned}
\label{eqn:two_point_test_bd}
\end{equation}
where \textnormal{(i)} follows from $A_+\cap A_-=\varnothing$, which implies $A_+\subseteq A_-^c$.

\paragraph{Step 3: deterministic Q-function and policy-value gaps.}
To apply \eqref{eqn:two_point_test_bd}, we need to verify that the optimal Q-functions of the two instances are sufficiently separated in the sup norm and span seminorm. For the policy-learning lower bound, we also need to show that choosing the action that is optimal in the opposite instance incurs a comparable value loss. Specifically, we will prove that, for every $0<\Delta\leq p/2$,
\begin{equation}
\norminf{q_{+,\Delta}^*-q_{-,\Delta}^*}
\geq
\frac{\Delta}{14(1-\gamma)p}
\quad \text{and}\quad
\spnorm{q_{+,\Delta}^*-q_{-,\Delta}^*}
\geq
\frac{\Delta}{12p^2}.\label{eqn:q_lb_separation}
\end{equation}
Moreover, for any fixed stationary deterministic policy $\pi$, if $\pi(0)=a_2$, then 
\begin{equation}
\norminf{v_{+,\Delta}^*-v_{+,\Delta}^{\pi}}
\geq
\frac{\Delta}{14(1-\gamma)p}.\label{eqn:v+_pi_gap}
\end{equation}
On the other hand, if $\pi(0)=a_1$, then
\begin{equation}
\norminf{v_{-,\Delta}^*-v_{-,\Delta}^{\pi}}
\geq
\frac{\Delta}{14(1-\gamma)p}.
\label{eqn:v-_pi_gap}
\end{equation}

For $x\in[p-\Delta,p+\Delta]$, let $v_x$ denote the discounted value function of the two-state Markov reward process with transition kernel $K_x$ from Step~1 and rewards $r(0)=0$ and $r(1)=1$. The value function $v_x$ uniquely solves $v_x = r+\gamma K_x v_x$, which gives
\[\
v_x(0)=
\frac{\gamma x}
{(1-\gamma)\crbk{1-\gamma+\gamma(p+x)}},\quad v_x(1)= \frac{1-\gamma+\gamma x}
{(1-\gamma)\crbk{1-\gamma+\gamma(p+x)}}.
\]
Moreover, note that
\[
\frac{d}{dx}v_x(0)
=
\frac{\gamma(1-\gamma+\gamma p)}
{(1-\gamma)\crbk{1-\gamma+\gamma(p+x)}^2}
>
0.
\]
Thus, increasing $x$ strictly improves the value at state $0$. Hence, $a_1$ is the unique optimal action at state $0$ for $\cM_{+,\Delta}$, whereas $a_2$ is  the unique optimal action at state $0$ in $\cM_{-,\Delta}$. Their optimal value functions are respectively $v_{p+\Delta}$ and $v_p$.

We note that
\[
\begin{aligned}
v_{p+\Delta}(0)-v_p(0)
&=
\frac{\gamma\Delta(1-\gamma+\gamma p)}
{(1-\gamma)\crbk{1-\gamma+\gamma(2p+\Delta)}
\crbk{1-\gamma+2\gamma p}},\\
v_p(0)-v_{p-\Delta}(0)
&=
\frac{\gamma\Delta(1-\gamma+\gamma p)}
{(1-\gamma)\crbk{1-\gamma+2\gamma p}
\crbk{1-\gamma+\gamma(2p-\Delta)}}.
\end{aligned}
\]
Since $1-\gamma\leq p$, $\gamma\geq0.9$, and $\Delta\leq p/2$, we have $1-\gamma +\gamma p\geq \gamma p$, $1-\gamma+2\gamma p\leq3p$,
$1-\gamma+\gamma(2p+\Delta)\leq\frac72p$, and $
1-\gamma+\gamma(2p-\Delta)\leq3p.$ It follows that
\[
v_{p+\Delta}(0)-v_p(0)
\geq
\frac{\Delta}{14(1-\gamma)p},
\qquad
v_p(0)-v_{p-\Delta}(0)
\geq
\frac{\Delta}{14(1-\gamma)p}.
\]

These inequalities prove the two wrong-action bounds in \eqref{eqn:v+_pi_gap} and \eqref{eqn:v-_pi_gap}. Indeed, every policy choosing the wrong action $a_2$ at state $0$ in $\cM_{+,\Delta}$ has value function $v_p$. Thus, 
\[
\pi(0)=a_2
\implies 
\norminf{v_{+,\Delta}^*-v_{+,\Delta}^{\pi}}= \norminf{v_{p+\Delta} - v_p}
\geq
\frac{\Delta}{14(1-\gamma)p}.
\]
Similarly, if the wrong action is chosen $a_1$ at state $0$ for $\cM_{-,\Delta}$, then the value function for this policy is $v_{p-\Delta}$. Therefore,
\[
\pi(0)=a_1
\implies
\norminf{v_{-,\Delta}^*-v_{-,\Delta}^{\pi}} = \norminf{v_{p} - v_{p-\Delta}}
\geq
\frac{\Delta}{14(1-\gamma)p}.
\]

We now turn to the Q-function separation lower bounds. For $\cM_{+,\Delta}$, since $a_1$ is optimal at 0, $q_{+,\Delta}^*(0,a_1)
=
v_{p+\Delta}(0)$; whereas for $\cM_{-,\Delta}$, action $a_2$ is optimal at state $0$, and hence $q_{-,\Delta}^*(0,a_1)\leq q_{-,\Delta}^*(0,a_2)=
v_p(0).$
Consequently,
\[
\norminf{q_{+,\Delta}^*-q_{-,\Delta}^*}
\geq
v_{p+\Delta}(0)-v_p(0)
\geq
\frac{\Delta}{14(1-\gamma)p}.
\]

It remains to verify the span separation. Note that
\begin{equation}  
\begin{aligned}
\spnorm{q_{+,\Delta}^*-q_{-,\Delta}^*} &= \max_{s,a}(q_{+,\Delta}^*(s,a)-q_{-,\Delta}^*(s,a)) - \min_{s,a}(q_{+,\Delta}^*(s,a)-q_{-,\Delta}^*(s,a))\\
&\geq
\crbk{q_{+,\Delta}^*(0,a_1)-q_{-,\Delta}^*(0,a_1)}
-
\crbk{q_{+,\Delta}^*(1,a_1)-q_{-,\Delta}^*(1,a_1)}
\end{aligned}
\label{eqn:q_span_decomp}
\end{equation}

From the sup norm bound, we have
\[
q_{+,\Delta}^*(0,a_1)-q_{-,\Delta}^*(0,a_1) \ge v_{p+\Delta}(0)-v_p(0)=\frac{\gamma\Delta(1-\gamma+\gamma p)}
{(1-\gamma)\crbk{1-\gamma+\gamma(2p+\Delta)}
\crbk{1-\gamma+2\gamma p}}
\]
On the other hand, at state $1$, both actions have the same reward and next-state distribution. Hence, the optimal Q-function agrees with the optimal value. In particular, $q_{+,\Delta}^*(1,a_1)  = v_{p+\Delta}(1),$ $q_{-,\Delta}^*(1,a_1)=v_p(1).$
Thus, applying the formula above gives
\[
q_{+,\Delta}^*(1,a_1) -q_{-,\Delta}^*(1,a_1)
=
\frac{\gamma^2p\Delta}
{(1-\gamma)\crbk{1-\gamma+\gamma(2p+\Delta)}
\crbk{1-\gamma+2\gamma p}}.
\]
Therefore, by \eqref{eqn:q_span_decomp}
\[
\spnorm{q_{+,\Delta}^*-q_{-,\Delta}^*} \geq
\frac{\gamma\Delta}
{\crbk{1-\gamma+\gamma(2p+\Delta)}
\crbk{1-\gamma+2\gamma p}}\geq
\frac{\Delta}{12p^2}.
\]
This proves \eqref{eqn:q_lb_separation}. 

\paragraph{Step 4: Concluding the lower bounds.}
We now choose $\Delta$ separately for the two statistical scales and apply the testing bound \eqref{eqn:two_point_test_bd}.

For sup-norm estimation and policy learning, set $\Delta:=56(1-\gamma)p\epsilon.$
The assumption $\epsilon\leq[112(1-\gamma)]^{-1}$ ensures that $\Delta\leq p/2$, so the two hard instances belong to $\cC(\gamma,\tmix)$. The sup-norm gap from \eqref{eqn:q_lb_separation} then implies
\[
\norminf{q_{+,\Delta}^*-q_{-,\Delta}^*}
\geq \frac{\Delta}{14(1-\gamma)p}= 
4\epsilon.
\]

Moreover, if the sample-size condition in part~\textnormal{(i)} of Theorem \ref{thm:lower_bounds} holds with $c_{\mathrm{lb}}=2^{-17}$, then, using $p=1/\tmix$,
\[
\frac{n\Delta^2}{p}
\le2^{12}
n(1-\gamma)^2p\epsilon^2
\leq
2^{-5}.
\]
Therefore, \eqref{eqn:two_point_test_bd} gives, for every $\cF_n$-measurable Q-function estimator $\tilde q_n$,
\[
\max_{\dagger\in\set{+,-}}
P^{\cM_{\dagger,\Delta}}
\crbk{\norminf{\tilde q_n-q_{\dagger,\Delta}^*}>\epsilon}
\geq
\frac12-\sqrt{2^{-4}}
=
\frac14.
\]
Since both instances belong to $\cC(\gamma,\tmix)$, taking the infimum over $\tilde q_n$ proves the Q-function statement in part~\textnormal{(i)}.

For policy learning, the same choice of $\Delta$ and the wrong-action gaps in Step~3 imply
\[
\begin{aligned}
P^{\cM_{+,\Delta}}\crbk{
\norminf{v_{+,\Delta}^*-v_{+,\Delta}^{\tilde\pi_n}}\ge 4\epsilon
}
&\geq
P^{\cM_{+,\Delta}}\crbk{\tilde\pi_n(0)= a_2},\\
P^{\cM_{-,\Delta}}\crbk{
\norminf{v_{-,\Delta}^*-v_{-,\Delta}^{\tilde\pi_n}}\ge 4\epsilon
}
&\geq
P^{\cM_{-,\Delta}}\crbk{\tilde\pi_n(0)=a_1}.
\end{aligned}
\]
Applying the same two-point testing calculation as in \eqref{eqn:two_point_test_bd} and relaxing the $4\epsilon$ to $\epsilon$ gives
\[
\begin{aligned}
\max_{\dagger\in\set{+,-}}
P^{\cM_{\dagger,\Delta}}\crbk{
\norminf{v_{\dagger,\Delta}^*-v_{\dagger,\Delta}^{\tilde\pi_n}}>\epsilon
}
&\geq
\frac12\left[
P^{\cM_{+,\Delta}}\crbk{\tilde\pi_n(0)= a_2}
+
P^{\cM_{-,\Delta}}\crbk{\tilde\pi_n(0)=a_1}
\right]\\
&\geq
\frac12\left[1-\crbk{P^{\cM_{+,\Delta}}\crbk{\tilde\pi_n(0)= a_1}-P^{\cM_{-,\Delta}}\crbk{\tilde\pi_n(0)=a_1}}
\right]\\
&\geq
\frac12\left[
1-
\normTV{\cL_{\cM_{+,\Delta}}(\mathbf S_n)-\cL_{\cM_{-,\Delta}}(\mathbf S_n)}
\right]\\
&\geq
\frac14.
\end{aligned}
\]
Taking the infimum over policy estimators proves the policy error part in \textnormal{(i)}.

For span-seminorm estimation, set $\Delta:=48p^2\epsilon.$
The assumption $\epsilon\leq\tmix/96=(96p)^{-1}$ ensures that $\Delta\leq p/2$. The span separation bound from \eqref{eqn:q_lb_separation} gives $\spnorm{q_{+,\Delta}^*-q_{-,\Delta}^*}\geq 4\epsilon.$

If the sample-size condition in part~\textnormal{(ii)} holds with $c_{\mathrm{lb}}=2^{-17}$, then we still have $n\Delta^2/p
\leq 2^{12}np^3\epsilon^2\leq 2^{-5}.$
Therefore, \eqref{eqn:two_point_test_bd}, now applied with the span seminorm, gives
\[
\max_{\dagger\in\set{+,-}}
P^{\cM_{\dagger,\Delta}}
\crbk{\spnorm{\tilde q_n-q_{\dagger,\Delta}^*}>\epsilon}
\geq
\frac14
\]
for every $\cF_n$-measurable Q-function estimator $\tilde q_n$. Taking the infimum over $\tilde q_n$ proves part~\textnormal{(ii)}. 
\end{proof}

\section{Conclusion}
\label{section:conclusion}

This paper develops a stability-sensitive sample-complexity theory for model-free discounted reinforcement learning. Without imposing communication, unichain, or mixing assumptions, we show that variance-reduced Q-learning adapts to the span $H=\spnorm{v^*}$ and achieves the minimax optimal sample complexity for sup-norm Q-function estimation and policy learning. Under uniform mixing, we further obtain a horizon-free leading sample complexity for Q-function estimation in the span seminorm. Matching lower bounds show that relative-Q estimation is statistically easier than policy learning in a minimax sense.

An important direction for future work is to use the span-error analysis to design an optimal model-free algorithm for average-reward RL. Since average-reward bias and relative-Q functions are defined only up to an additive constant, the span seminorm is the natural error measure. Extending the variance-reduction and signed-noise propagation arguments developed here may therefore avoid the artificial horizon dependence introduced by discounted reductions.

In addition, while we obtain parameter-free implementations for sup-norm Q-estimation and policy learning without sacrificing optimality, an analogous result for span-seminorm Q-estimation remains open. Such an algorithm would be particularly valuable for average-reward RL, but would likely require new design principles because the uniform mixing-time upper bound $\tmix$ cannot be inferred from an estimate of the optimal Q-function alone.

\section*{Acknowledgements}
J. Blanchet gratefully acknowledges support from Air Force Office of Scientific Research under award FA9550-20-1-0397, ONR under award N00014-24-1-2655, and the National Science Foundation (NSF) under grants 1915967, 2118199, 2312204 and 2403007. 

\bibliographystyle{apalike}
\bibliography{ref}

\appendix
\appendixpage

\section{Proof of the Q-Learning Convergence Bound}
\label{appendix:q_learning_span_bound}

\begin{proof}[Proof of Proposition~\ref{prop:q_learning_span_convergence}]
The finite-time Q-learning estimate in Section~4.1.2 of \citet{wainwright2019} implies that there exists a universal constant $c\geq1$ such that, with probability at least $1-\delta$,
\begin{align}
\norminf{q_{k_0+1}-q^*}
&\leq \frac{\norminf{q^*}}{(1-\gamma)k_0}
+c\frac{\norminf{\sigma^*}}{(1-\gamma)^{3/2}\sqrt{k_0}}
\sqrt{\log\crbk{\frac{2|S||A|k_0}{\delta}}}\notag\\
&\quad
+c\frac{\spnorm{q^*}}{(1-\gamma)^2k_0}
\log\crbk{\frac{2e|S||A|k_0(1+(1-\gamma)k_0)}{\delta}}.
\label{eqn:wainwright_q_learning_bound_appendix}
\end{align}
We bound the three instance-dependent quantities in \eqref{eqn:wainwright_q_learning_bound_appendix} in terms of $H$.

First, bounded rewards imply
\begin{equation}
\label{eqn:qstar_inf_bound_appendix}
\norminf{q^*}\leq\frac{1}{1-\gamma}.
\end{equation}
Second, because $v^*$ has range length $H$, for every $(s,a)\in S\times A$,
\begin{equation}
\label{eqn:sigma_star_H_bound_appendix}
(\sigma^*)^2(s,a)
=\gamma^2\var_{S'\sim P_{s,a}}(v^*(S'))
\leq\frac{\gamma^2H^2}{4}.
\end{equation}
Thus $\norminf{\sigma^*}\leq\gamma H/2$.

Finally, for any two state-action pairs $(s,a)$ and $(s',a')$,
\[\begin{aligned}
q^*(s,a)-q^*(s',a')
&=r(s,a)-r(s',a')
+\gamma\crbk{P_{s,a}[v^*]-P_{s',a'}[v^*]}\\
&\leq 1+\gamma H.
\end{aligned}\]
Interchanging the two state-action pairs gives the reverse inequality, and therefore
\begin{equation}
\label{eqn:qstar_span_H_bound_appendix}
\spnorm{q^*}\leq1+\gamma H.
\end{equation}

Substituting \eqref{eqn:qstar_inf_bound_appendix}--\eqref{eqn:qstar_span_H_bound_appendix} into \eqref{eqn:wainwright_q_learning_bound_appendix} gives
\[\begin{aligned}
\norminf{q_{k_0+1}-q^*}
&\leq \frac{1}{(1-\gamma)^2k_0}
+\frac{c\gamma H}{2(1-\gamma)^{3/2}}
\sqrt{\frac{1}{k_0}\log\crbk{\frac{2|S||A|k_0}{\delta}}}\\
&\quad
+c\frac{1+\gamma H}{(1-\gamma)^2k_0}
\log\crbk{\frac{2e|S||A|k_0(1+(1-\gamma)k_0)}{\delta}}.
\end{aligned}\]
The logarithm in the last term is at least one and dominates the logarithm in the square-root term. Since $1+\gamma H\geq1$, the first term can also be absorbed into the last term after increasing the universal constant. This proves \eqref{eqn:q_learning_span_convergence}.
\end{proof}

\section{Proof of Lemma \ref{lemma:quantitative_inner_loop_bound}}\label{section:proof:lemma:quantitative_inner_loop_bound}
\begin{proof}
The proof is adapted from \citet{wainwright2019}, Lemma 3. 

Conditional on the recentering samples that form $\widetilde\cT_l(q^\circ)$, the operators $\widehat\cJ_{k+1}^\circ$ are independent, monotone, and $\gamma$-contractive, with conditional expectation $\cJ^\circ$. Subtracting the fixed-point, we see that
$$
q_{k+1}^\circ-\bar q^\circ
=(1-\lambda_k)(q_k^\circ-\bar q^\circ)
+\lambda_k\crbk{\widehat\cJ_{k+1}^\circ(q_k^\circ)-\widehat\cJ_{k+1}^\circ(\bar q^\circ)}
+\lambda_k\xi_{k+1},
$$
where 
$$\xi_{k+1}
:=\crbk{\widehat\cT_{k+1}(\bar q^\circ)-\widehat\cT_{k+1}(q^\circ)}
-\crbk{\cT(\bar q^\circ)-\cT(q^\circ)}$$
has mean zero. Moreover, the empirical and population Bellman operators are both $\gamma$-contractive, so $\norminf{\xi_{k+1}}\leq2\gamma\norminf{q^\circ-\bar q^\circ}$.

Let $R_1=0$ and define $R_{k+1}=(1-\lambda_k)R_k+\lambda_k\xi_{k+1}$. The cone-contractive comparison in Corollary~1 of \citet{wainwright2019l_infty}, specialized exactly as in Appendix~B.1 of \citet{wainwright2019}, gives the pathwise bound
\begin{equation}
\norminf{q_{K+1}^\circ-\bar q^\circ}
\leq\frac{2}{1+(1-\gamma)K}\left(\norminf{q^\circ-\bar q^\circ}+\sum_{k=1}^K\norminf{R_k}\right)+\norminf{R_{K+1}}.
\label{eqn:conditional_inner_loop_cone_bound}
\end{equation}
The noise estimate in Lemma~2 of \citet{wainwright2019} yields a universal constant $c_1$ such that, with conditional probability at least $1-\eta$,
\begin{equation}
\frac{2\sum_{k=1}^K\norminf{R_k}}{1+(1-\gamma)K}+\norminf{R_{K+1}}
\leq\frac{c_1\norminf{q^\circ-\bar q^\circ}}{1-\gamma}\sqrt{\frac{\log\crbk{8|S||A|K/\eta}}{1+(1-\gamma)K}},
\label{eqn:conditional_inner_loop_noise_bound}
\end{equation}
provided that $K\geq \log\crbk{8|S||A|K/\eta}$.

Combining \eqref{eqn:conditional_inner_loop_cone_bound} and \eqref{eqn:conditional_inner_loop_noise_bound}, we have
\[\begin{aligned}
\norminf{q_{K+1}^\circ-\bar q^\circ}
&\leq\norminf{q^\circ-\bar q^\circ}\left[\frac{2\sqrt{\log\crbk{8|S||A|K/\eta}}}{1+(1-\gamma)K}+\frac{c_1}{1-\gamma}\sqrt{\frac{\log\crbk{8|S||A|K/\eta}}{1+(1-\gamma)K}}\right]\\
&\leq c\norminf{q^\circ-\bar q^\circ}\left[\frac{\sqrt{\log\crbk{8|S||A|K/\eta}}}{1+(1-\gamma)K}+\sqrt{\frac{\log\crbk{8|S||A|K/\eta}}{(1-\gamma)^3K}}\right].
\end{aligned}\]
This proves \eqref{eqn:conditional_inner_loop_bound}.
\end{proof}

\section{Proofs for the Results in Section~\ref{section:span_seminorm_uniform_mixing}}
\label{section:proof_span_uniform_mixing}

Before establishing the proof, it is convenient to transfer the mixing-time bound in Assumption~\ref{assump:uniform_mixing} from the state kernel $P_\pi$ to the lifted state-action kernel $P^\pi$. Let $\eta_\pi$ be the stationary distribution of $P_\pi$, and define $\widetilde\eta_\pi(s,a):=\eta_\pi(s)\1\set{a=\pi(s)}$. A direct calculation shows that $\widetilde\eta_\pi$ is stationary for $P^\pi$. Moreover, for every $k\geq1$,
\[
\normTV{(P^\pi)^k((s,a),\cd)-\widetilde\eta_\pi}
=\normTV{P_{s,a}P_\pi^{k-1}-\eta_\pi}.
\]
For any probability distribution $\mu$ on $S$, convexity of total variation gives
\[
\begin{aligned}
\normTV{\mu P_\pi^{k-1}-\eta_\pi}
&=\normTV{\sum_{s_0\in S}\mu(s_0)\crbk{P_\pi^{k-1}(s_0,\cd)-\eta_\pi}}\\
&\leq\sum_{s_0\in S}\mu(s_0)\normTV{P_\pi^{k-1}(s_0,\cd)-\eta_\pi}\\
&\leq\sup_{s_0\in S}\normTV{P_\pi^{k-1}(s_0,\cd)-\eta_\pi}.
\end{aligned}
\]
Taking $\mu=P_{s,a}$ shows that $\tmix(P^\pi)\leq\tmix(P_\pi)+1\leq2\tmix$, where the last inequality uses $\tmix\geq1$.

The proof follows the same epoch-wise structure as the proof of Theorem~\ref{thm:pathwise_q_and_policy_error}, where the key variance estimate is replaced by a more intricate span analogue in Proposition \ref{prop:signed_propagated_std_uniform_mixing}. 

\subsection{Span Variance Propagation Through Bellman Fixed-Point}

Unlike in the sup-norm error analysis, we need to consider the propagation through the Bellman recursion of any \textit{signed} vector that is entry-wise bounded within $\pm\sigma(q)$. This is the form needed to bound the variance after applying the entry-wise Bernstein inequality in the epoch proof.

\begin{proposition}[Span propagation bound]
\label{prop:signed_propagated_std_uniform_mixing}
Suppose Assumption~\ref{assump:uniform_mixing} holds. Let $\bar r:S\times A\ra\R$ be arbitrary, let $\bar q$ be the optimal action-value function of the MDP with reward $\bar r$ and transition kernel $P$, and let $\pi$ be an optimal policy for this MDP. Then, for every $q\in\R^{|S||A|}$,
\[
\sup_{\abs{h}\leq\sigma(q)}\spnorm{(I-\gamma P^\pi)\inv h}
\leq8\sqrt{6}\spnorm{\bar r}\tmix^{3/2}
+4\tmix\spnorm{q-\bar q},
\]
where the inequality $\abs{h}\leq\sigma(q)$ is understood entry-wise. Moreover,
\[
\sup_{\norminf{h}\leq1}\spnorm{(I-\gamma P^\pi)\inv h}
\leq8\tmix.
\]
\end{proposition}

\begin{proof}
Fix any policy $\pi\in\Pi$ and write $K:=P^\pi$ and $m:=\tmix(K)$.
By the lifted-kernel mixing bound established before the proposition, $m\leq2\tmix$.

Let $\eta$ be the stationary distribution of $K$. By the definition of $m$,
\[
\sup_{x\in S\times A}\normTV{K^m(x,\cd)-\eta}\leq\frac14.
\]
Therefore, applying the triangle inequality, we see that for every $x,y\in S\times A$,
\[
\normTV{K^m(x,\cd)-K^m(y,\cd)} = \normTV{(K^m(x,\cd)-\eta) + (\eta-K^m(y,\cd))}
\leq\frac12.
\]
Thus, for every $h:S\times A\ra \R$,
\begin{equation}
\label{eqn:m_step_span_contraction}
\spnorm{K^{jm}h}
\leq\frac12\spnorm{ K ^{(j-1)m}h}\leq\ds\leq 2^{-j}\spnorm{h} .
\end{equation}
Define the one-block operator
\[
B_m:=\sum_{i=0}^{m-1}\gamma^iK^i.
\]
Then, grouping the Neumann series $(I-\gamma K)\inv = \sum_{k=0}^\infty \gamma^k K^k$ into blocks of length $m$ gives
\begin{equation}
\begin{aligned}
\spnorm{(I-\gamma K)\inv h}&\leq\sum_{j=0}^\infty\gamma^{jm}\spnorm{K^{jm}B_mh}\\
&\leq\sum_{j=0}^\infty2^{-j}\spnorm{B_mh}\\
&\leq 2\spnorm{B_mh}.
\end{aligned}\label{eqn:tu_Bmh_bd_inv}
\end{equation}
where the second inequality uses \eqref{eqn:m_step_span_contraction}.

Since the stochastic matrix $K$ is non-expansive in the span seminorm,
\[
\spnorm{B_mh}
\leq\sum_{i=0}^{m-1}\gamma^i\spnorm{K^ih}\leq m\spnorm{h}\leq 2m\norminf{h}
\]
Thus, by \eqref{eqn:tu_Bmh_bd_inv}, $\spnorm{(I-\gamma K)\inv h}\leq2\spnorm{B_mh}\leq4m\norminf{h}.$
Taking the supremum over $\norminf{h}\leq1$ and noting that $m\leq 2\tmix$ gives
\begin{equation}
\sup_{\norminf{h}\leq1}\spnorm{(I-\gamma P^\pi)\inv h}
\leq8\tmix.\label{eqn:tu_second_part_span_var_prop}
\end{equation}
This proves the second claim in Proposition \ref{prop:signed_propagated_std_uniform_mixing}. 

We then prove the first statement. By \eqref{eqn:tu_Bmh_bd_inv}, it suffices to bound $\spnorm{B_mh}$ for $|h|\leq \sigma( q)$. Note that for such $h$ and every $x\in S\times A$, since $K$ is nonnegative and $K^i$ is a stochastic kernel, applying Jensen's inequality twice gives
\begin{equation}
\begin{aligned}
\abs{B_mh(x)}
&\leq m\frac{1}{m}\sum_{i=0}^{m-1}\gamma^iK^i[\sigma(q)](x)\\
&\leq m\sqrt{\frac{1}{m}\sum_{i=0}^{m-1}\gamma^{2i}\crbk{K^i[\sigma( q)](x)}^2}\\
&\leq \sqrt m \sqrt{\sum_{i=0}^{m-1}\gamma^{2i}K^i[\sigma( q)^2](x)}
\end{aligned}\label{eqn:tu_Bmh_bd}
\end{equation}

We start by considering the special case $q = \bar q$; i.e. we show that
\begin{equation}
\label{eqn:propagated_sigma_bar_q_bound}
\sup_{\abs{h}\leq\sigma(\bar q)}
\spnorm{(I-\gamma P^\pi)\inv h}
\leq 8\sqrt{6}\spnorm{\bar r}\tmix^{3/2}.
\end{equation}
Since $\pi$ is the optimal policy under the reward $\bar r$, we have 
$\bar q=(I-\gamma K)\inv\bar r$. Hence \eqref{eqn:tu_Bmh_bd_inv} gives
\begin{equation}
\spnorm{\bar q}
\leq2\spnorm{B_m\bar r}\leq2m\spnorm{\bar r}.
\label{eqn:span_bar_q_block_bound}
\end{equation}
Since we are bounding variance-related quantities, it is useful to consider a mid-point centered version of $\bar q$. Specifically, let
\[
d:=\frac{\max_x\bar q(x)+\min_x\bar q(x)}{2}\quad \text{and}\quad \bar q^c := \bar q -d.
\] 
We also consider a centered reward $\bar r^c:= \bar r-(1-\gamma)d$.

It is easy to verify that by \eqref{eqn:span_bar_q_block_bound} and the Bellman equation,
\begin{equation}
\label{eqn:centered_q_reward_bounds}
\bar q^c=\bar r^c+\gamma K\bar q^c,
\quad
\norminf{\bar q^c}\leq m\spnorm{\bar r},
\quad
\norminf{\bar r^c}\leq \spnorm{\bar r}.
\end{equation}
Because $\pi$ is optimal under $\bar r$ and hence $\bar r^c$, applying the Bellman equation in \eqref{eqn:centered_q_reward_bounds}, we have
\[
\begin{aligned}
\sigma(\bar q)^2&= \sigma(\bar q^c)^2\\
&=\gamma^2K[(\bar q^c)^2]-\gamma^2K[\bar q^c]^2\\
&=\gamma^2K[(\bar q^c)^2]-(\bar q^c-\bar r^c)^2\\
&=\gamma^2K[(\bar q^c)^2]-(\bar q^c)^2+2\bar q^c\bar r^c-(\bar r^c)^2.
\end{aligned}
\]
Multiplying this identity by $\gamma^{2i}K^i$ and summing from $i=0$ to $m-1$ yields
\[\begin{aligned}
\sum_{i=0}^{m-1}\gamma^{2i}K^i[\sigma(\bar q)^2]&=\sum_{i=0}^{m-1}\gamma^{2(i+1)}K^{i+1}[(\bar q^c)^2]-\sum_{i=0}^{m-1}\gamma^{2i}K^{i}[(\bar q^c)^2]+\sum_{i=0}^{m-1}\gamma^{2i}K^i\crbk{2\bar q^c\bar r^c-(\bar r^c)^2}\\
&=\gamma^{2m}K^m[(\bar q^c)^2]-(\bar q^c)^2+\sum_{i=0}^{m-1}\gamma^{2i}K^i\crbk{2\bar q^c\bar r^c-(\bar r^c)^2}\\
&\leq \gamma^{2m}K^m[(\bar q^c)^2]
+2\sum_{i=0}^{m-1}\gamma^{2i}K^i[\abs{\bar q^c\bar r^c}]\\
&\leq \norminf{\bar q^c}^2
+2m\norminf{\bar q^c}\norminf{\bar r^c}\\
&\leq 3m^2\spnorm{\bar r}^2.
\end{aligned}
\]
where we applied \eqref{eqn:centered_q_reward_bounds} to bound $\norminf{\bar q^c}$ and $\norminf{\bar r^c}$. Consequently, from \eqref{eqn:tu_Bmh_bd_inv}, \eqref{eqn:tu_Bmh_bd}, and $m\leq2\tmix$,
\[
\spnorm{(I-\gamma K)\inv h} \le 2\spnorm{B_mh}
\leq 4\norminf{B_mh}
\leq 8\sqrt{6}\,\spnorm{\bar r}\tmix^{3/2}.
\]
Taking the supremum over $\abs{h}\leq\sigma(\bar q)$ proves \eqref{eqn:propagated_sigma_bar_q_bound}.

It remains to replace $\bar q$ by an arbitrary $q$. Define
\[
\Delta(s):=v(q)(s)-v(\bar q)(s),
\quad \forall s\in S.
\]
Then, we have that for all $s\in S$, 
\[
\min_{a\in A}\crbk{q(s,a)-\bar q(s,a)}
\leq \Delta(s)
\leq \max_{a\in A}\crbk{q(s,a)-\bar q(s,a)}.
\]
Taking the maximum and minimum over $s$ gives
\[
\spnorm{\Delta}
\leq \spnorm{q-\bar q}.
\]
For every $(s,a)\in S\times A$, by the triangle inequality for standard deviations,
\[
\begin{aligned}
\sigma(q)(s,a) &= \gamma \sqrt{\var_{s,a}(v(\bar q)+v(q)-v(\bar q))}\\
&\leq \sigma(\bar q)(s,a)
+\gamma\sqrt{\var_{s,a}\crbk{\Delta}}\\
&\leq \sigma(\bar q)(s,a)
+\frac12\spnorm{q-\bar q}.
\end{aligned}
\]
Consequently, every $h$ satisfying $\abs{h}\leq\sigma(q)$ can be written as $h=h_1+h_2$, where entry-wise
\[
\abs{h_1}\leq\sigma(\bar q)
\qquad\text{and}\qquad
\abs{h_2}\leq\frac12\spnorm{q-\bar q}.
\]
Using the triangle inequality for the span seminorm, \eqref{eqn:tu_second_part_span_var_prop}, and \eqref{eqn:propagated_sigma_bar_q_bound}, we obtain
\[
\begin{aligned}
\spnorm{(I-\gamma P^\pi)\inv h}
&\leq \spnorm{(I-\gamma P^\pi)\inv h_1}
+\spnorm{(I-\gamma P^\pi)\inv h_2}\\
&\leq  8\sqrt{6}\spnorm{\bar r}\tmix^{3/2}
+\frac12\spnorm{q-\bar q}
\sup_{\norminf{g}\leq1}\spnorm{(I-\gamma P^\pi)\inv g}\\
&\leq 8\sqrt{6}\spnorm{\bar r}\tmix^{3/2}
+4\tmix\spnorm{q-\bar q}.
\end{aligned}
\]
Taking the supremum over $\abs{h}\leq\sigma(q)$ proves the first claim in Proposition \ref{prop:signed_propagated_std_uniform_mixing}.
\end{proof}

\subsection{Proof of Theorem~\ref{thm:pathwise_span_error_uniform_mixing}}

We now prove Theorem~\ref{thm:pathwise_span_error_uniform_mixing} using a strategy similar to the proof of Theorem~\ref{thm:pathwise_q_and_policy_error}. The principal difference is the control of the perturbed fixed-point error $\spnorm{\bar q_l-q^*}$. The direct Bellman comparison used in the sup-norm proof does not yield a span bound because it involves propagated variance associated with different policies. We instead interpolate between the original and perturbed rewards and apply Proposition~\ref{prop:signed_propagated_std_uniform_mixing} along intervals on which one policy remains optimal.

\begin{proof}
For $0\leq l\leq l^*$, define the deterministic target levels
\begin{equation}
\label{eqn:def_geometric_span_target_b_l}
b_l:=2^{-l}\cd8\tmix.
\end{equation}
Fix an original optimal policy $\pi^*$. Since $q^*=(I-\gamma P^{\pi^*})\inv r$ and $\norminf{r}\leq1$, the second assertion of Proposition~\ref{prop:signed_propagated_std_uniform_mixing} gives
\[
\spnorm{\hat q_0-q^*}
=\spnorm{q^*}
\leq8\tmix
=b_0.
\]
Thus, the initial target holds deterministically. We show that, conditional on $\spnorm{\hat q_{l-1}-q^*}\leq b_{l-1}$, epoch $l$ reaches the target $b_l$ except on an event of conditional probability at most $\delta/l^*$.

\paragraph{Error of the perturbed fixed point.}
Fix an epoch $l\in\set{1,\ldots,l^*}$, condition on $\cG_{l-1}$, and suppose
\begin{equation}
\label{eqn:span_epoch_induction_hypothesis}
\spnorm{\hat q_{l-1}-q^*}
\leq b_{l-1}.
\end{equation}
Recall that
\[
z_l:=\bar r_l-r
=\widetilde\cT_l(\hat q_{l-1})-\cT(\hat q_{l-1}).
\]
As in the proof of Theorem \ref{thm:pathwise_q_and_policy_error}, we first need an entry-wise bound for $|z_l|$. For a fixed coordinate $(s,a)$, the $n_l$ random variables averaged in $z_l(s,a)$ are conditionally independent and centered, have conditional variance $\sigma^2(\hat q_{l-1})(s,a)$, and are bounded in absolute value by $\gamma\spnorm{\hat q_{l-1}}$. Thus, we apply Bernstein's inequality and a union bound over $S\times A$ to obtain the entry-wise estimate
\begin{equation}
\label{eqn:span_recenter_concentration}
\abs{z_l}
\leq\sqrt{\frac{2\log\crbk{8|S||A|l^*/\delta}}{n_l}}\,\sigma(\hat q_{l-1})
+\frac{2\gamma\spnorm{\hat q_{l-1}}\log\crbk{8|S||A|l^*/\delta}}{3n_l}
\end{equation}
with $\cG_{l-1}$-conditional probability at least $1-\delta/(2l^*)$ on the event in \eqref{eqn:span_epoch_induction_hypothesis}.

The initial argument gives $\spnorm{q^*}\leq8\tmix$, and $b_{l-1}\leq b_0=8\tmix$. Hence, under \eqref{eqn:span_epoch_induction_hypothesis},
\begin{equation}\label{eqn:tu_ql-1_tmix_bd}
\spnorm{\hat q_{l-1}}
\leq\spnorm{q^*}+\spnorm{\hat q_{l-1}-q^*}
\leq16\tmix.
\end{equation}
Moreover, the range of $v(\hat q_{l-1})$ is at most $\spnorm{\hat q_{l-1}}$, so Popoviciu's inequality gives
\[
\norminf{\sigma(\hat q_{l-1})}
\leq\frac{\gamma}{2}\spnorm{v(\hat q_{l-1})}
\leq\frac12\spnorm{\hat q_{l-1}}
\leq8\tmix.
\]
Since the batch-size choice in \eqref{eqn:n_l_span_uniform_mixing_choice} implies
\[
n_l
\geq C\tmix^2\log\crbk{\frac{8|S||A|l^*}{\delta}},
\]
on the event in \eqref{eqn:span_recenter_concentration}
\begin{equation}
\norminf{z_l}
\leq8\sqrt{\frac{2}{C}}+\frac{32}{3C\tmix}
\leq2\label{eqn:tu_zl_norm_bd}
\end{equation}
for a sufficiently large universal constant $C$. 

To bound $ \spnorm{\bar q_l - q^*}$, we note that the bound \eqref{eqn:bar_q_comparison} no longer works for the span seminorm case. Indeed, in general we only have
$$ (I-\gamma P^{\pi^*})\inv z_l\leq \bar q_l - q^*\leq(I-\gamma P^{\bar\pi_l})\inv z_l $$
where the two policies are not the same and $\bar\pi_l$ correlates with $z_l$. 

To resolve this policy mismatch issue in the bound, we introduce an interpolation technique that connects $r$ and the perturbed reward $r+z_l$ trough the one-parameter family of rewards $\set{r+tz_l:t\in[0,1]}$. Specifically, for $t\in[0,1]$, define $r_{l,t}:=r+t z_l$, and let $q_{l,t}$ be the optimal action-value function under reward $r_{l,t}$. Thus, $q_{l,0}=q^*$ and $q_{l,1}=\bar q_l.$

Important to this construction is the following property. First, for any fixed policy $\pi$, its policy action-value function under reward $r_{l,t}$ is
\[
q_{l,t}^\pi
:=(I-\gamma P^\pi)\inv(r+t z_l),
\]
which is affine in $t$. The policy $\pi$ is optimal at time $t$ if and only if it is greedy with respect to $q_{l,t}^\pi$. Thus, the set of times such that a fixed policy $\pi\in\Pi$ is optimal is given by the feasible set of the inequalities
$$\set{t\in[0,1]: (I-\gamma P^\pi)\inv(r+t z_l)(s,\pi(s))\geq (I-\gamma P^\pi)\inv(r+t z_l)(s,a),\; \forall (s,a)\in S\times A}.$$ 
Since the constraints are linear in $t$, this set must either be a closed interval or empty. 

Since $\Pi$ is finite and at least one deterministic policy is optimal at each time, the endpoints of these intervals induce a finite partition
\[
0=t_0<t_1<\cdots<t_J=1
\]
and policies $\pi_1,\ldots,\pi_J$ such that $\pi_j$ is optimal throughout $[t_{j-1},t_j]$. On $[t_{j-1},t_j]$, optimality of $\pi_j$ gives $q_{l,t}=q_{l,t}^{\pi_j}$, and therefore
\begin{equation}
    \label{eqn:tu_q_interval_bd}
q_{l,t_j}-q_{l,t_{j-1}}
=(t_j-t_{j-1})(I-\gamma P^{\pi_j})\inv z_l.
\end{equation}
To use this construction to control the perturbed fixed-point error we define
\[
\psi_l
:=\sup_{0\leq t\leq1}\spnorm{q_{l,t}-q^*}\geq \spnorm{\bar q_l - q^*}.
\]
where the second inequality follows from $q_{l,1} = \bar q_l$. Therefore, to bound the perturbed fixed-point error, it suffices to upper bound $\psi_l$.

For any $t\in[t_{j-1},t_{j}]$, we have
$$q_{l,t} - q^* = q_{l,t} -q_{l,t_{j-1}} +\sum_{k=1}^{j-1}(q_{l,t_k}-q_{l,t_{k-1}}).$$ Using \eqref{eqn:tu_q_interval_bd} and subadditivity of the span seminorm, we have that
\[
\psi_l
\leq\max_{1\leq j\leq J}\spnorm{(I-\gamma P^{\pi_j})\inv z_l}.
\]
We then employ the following Lemma~\ref{lemma:span_recentered_perturbation_bound} to bound the right-hand side. To streamline the main argument, the proof is deferred to the next section.

\begin{lemma}
\label{lemma:span_recentered_perturbation_bound}
Suppose \eqref{eqn:tu_ql-1_tmix_bd} and \eqref{eqn:span_recenter_concentration} hold. Let $r^\circ:S\times A\ra\R$ satisfy $\spnorm{r^\circ}\leq5$, let $q^\circ$ and $\pi^\circ$ be the optimal Q-function and policy under $r^\circ$. Then there exist universal constants $C_1,C_2\geq1$ such that
\[
\spnorm{(I-\gamma P^{\pi^\circ})\inv z_l}
\leq C_1(\tmix^{3/2}+\tmix\spnorm{\hat q_{l-1}-q^\circ})\sqrt{\frac{\log\crbk{8|S||A|l^*/\delta}}{n_l}}+C_2\frac{\tmix^2\log\crbk{8|S||A|l^*/\delta}}{n_l}.
\]
\end{lemma}

To apply Lemma \ref{lemma:span_recentered_perturbation_bound}, we note that by \eqref{eqn:tu_zl_norm_bd}, we have that for every $t\in[0,1]$,
\[
\spnorm{r+t z_l}
\leq\spnorm{r}+t\spnorm{z_l}
\leq1+2\norminf{z_l}
\leq5.
\]
Moreover, for each $j$, choose any $t\in[t_{j-1},t_j]$. Since $\pi_j$ is optimal for the reward $r_{l,t}$,
\[
\spnorm{\hat q_{l-1}-q_{l,t}}
\leq\spnorm{\hat q_{l-1}-q^*}+\spnorm{q^*-q_{l,t}}
\leq b_{l-1}+\psi_l.
\]

Thus, applying Lemma \ref{lemma:span_recentered_perturbation_bound} with $ r^\circ=r_{l,t}$ and $ q^\circ=q_{l,t}$, and then maximizing over the interpolation intervals, yields
\begin{equation}
\label{eqn:psi_l_self_bounding_span}
\begin{aligned}
\psi_l\leq C_1(\tmix^{3/2}+\tmix(b_{l-1}+\psi_l))\sqrt{\frac{\log\crbk{8|S||A|l^*/\delta}}{n_l}}
+C_2\frac{\tmix^2\log\crbk{8|S||A|l^*/\delta}}{n_l}
\end{aligned}
\end{equation}
whenever \eqref{eqn:tu_ql-1_tmix_bd} and \eqref{eqn:span_recenter_concentration} hold. Since
\[
n_l
\geq C(4^l\tmix +\tmix^2)\log\crbk{\frac{8|S||A|l^*}{\delta}},
\]
we have
\[
\tmix^{3/2}\sqrt{\frac{\log\crbk{8|S||A|l^*/\delta}}{n_l}}
\leq\frac{\tmix}{2^l\sqrt C}
=\frac{b_l}{8\sqrt C}, \quad b_{l-1}\tmix\sqrt{\frac{\log\crbk{8|S||A|l^*/\delta}}{n_l}}
\leq\frac{2b_{l}}{\sqrt C},
\]
and
\[
\frac{\tmix^2\log\crbk{8|S||A|l^*/\delta}}{n_l}
\leq\frac{\tmix}{C4^l}
=\frac{b_l}{8C2^l}
\leq\frac{b_l}{16C}.
\]
Therefore, \eqref{eqn:psi_l_self_bounding_span} implies
\[
\psi_l
\leq\left(\frac{C_1}{8\sqrt C}+\frac{2C_1}{\sqrt C}+\frac{C_2}{16C}\right)b_l
+\frac{C_1}{\sqrt C}\psi_l.
\]

We can choose a sufficiently large constant $C$ so that $C_1/\sqrt C\leq1/4$ and the coefficient multiplying $b_l$ is at most $3/10$. Then, rearranging the inequality gives with probability at least $1-\delta/(2l^*)$
\begin{equation}
\label{eqn:bar_q_span_epoch_bound}
\spnorm{\bar q_l-q^*}
\leq\psi_l
\leq\frac25b_l
\end{equation}
conditional on $\cG_{l-1}$ and on the event \eqref{eqn:span_epoch_induction_hypothesis}.

\paragraph{Inner-loop error bound.}
To bound the inner-loop last iterate error, we first consider a centered version of the update. Let $\cH_l$ be the sigma-field generated by $\cG_{l-1}$ and the recentering samples used in epoch $l$. Consider running an auxiliary recursion from $\hat q_{l-1}^c = \hat q_{l-1} - d_l$, where $d_l$, conditioned on $\cH_{l}$, is a constant to be defined, using the same empirical operators as the actual recursion, and let $q_{l,k}^c$ denote its iterates. If $q_{l,k}^c=q_{l,k}-a_k$ for some scalar $a_k$, then constant-shift equivalence of the recentered operator gives
\[
q_{l,k+1}^c
=q_{l,k+1}-a_{k+1}
\quad\text{with}\quad 
a_{k+1}=\crbk{1-(1-\gamma)\lambda_k}a_k.
\]
Since $a_1=d_l$ is constant, induction shows that the actual and auxiliary iterates differ by a constant at every iteration. In particular,
\begin{equation}\label{eqn:tu_q_span_err_shift_equiv}
\spnorm{\hat q_l-\bar q_l}
=\spnorm{q_{l,k^*+1}^c-\bar q_l}.
\end{equation}
This implies that it suffices to analyze the span error of $\spnorm{q_{l,k^*+1}^c-\bar q_l}$ starting from any constant shift $d_l$.

Thus, we can choose $d_l$ to be the mid-point of $\hat q_{l-1} - \bar q_l$ and hence 
\begin{equation}
    \norminf{\hat q_{l-1}^c -\bar q_l} = \frac{1}{2}\spnorm{\hat q_{l-1} - \bar q_l}.\label{eqn:tu_choice_of_dl}
\end{equation}

Then, to bound $ \norminf{q_{l,k^*+1}^c-\bar q_l}$, we recall Lemma \ref{lemma:quantitative_inner_loop_bound}. Note that the operator $\widehat \cJ_{l,k+1}$ and $\cJ_l$ in \eqref{eqn:def_recentered_operator_span_dependent} are invariant to a shift in the recentering anchor; i.e.,
\[
\begin{aligned}
\widehat \cJ_{l,k+1}(q)&=\widehat\cT_{l,k+1}(q)-\widehat\cT_{l,k+1}(\hat q_{l-1})+\widetilde\cT_l(\hat q_{l-1})\\
&=\widehat\cT_{l,k+1}(q)-\widehat\cT_{l,k+1}(\hat q_{l-1}^c)+\widetilde\cT_l(\hat q_{l-1}^c)
\end{aligned}
\]
and the same holds for $\cJ_l$. 
Thus, the recentered random operator, its conditional expectation, and its fixed point $\bar q_l$ are unchanged by replacing the anchor with $\hat q_{l-1}^c$. Moreover, there is a universal constant $c_0$ such that
\begin{equation}
\label{eqn:span_inner_log_adjustment}
\log\crbk{\frac{16|S||A|k^*l^*}{\delta}}
\leq c_0\log\crbk{\frac{8|S||A|l^*}{(1-\gamma)\delta}}.
\end{equation}
Hence, for sufficiently large $C$, $k^*\geq\log\crbk{16|S||A|k^*l^*/\delta}$. 

Thus, conditioned on $\cH_l$, Lemma~\ref{lemma:quantitative_inner_loop_bound} applies with the anchor $\hat q_{l-1}^c$, failure probability $\eta=\delta/(2l^*)$, and $K=k^*$. We then conclude that with $\cH_l$-conditional probability at least $1-\delta/(2l^*)$,
\[
\begin{aligned}
\norminf{q_{l,k^*+1}^c-\bar q_l}\leq c\norminf{\hat q_{l-1}^c-\bar q_l}\left[
\frac{\sqrt{\log\crbk{16|S||A|k^*l^*/\delta}}}{1+(1-\gamma)k^*}
+\sqrt{\frac{\log\crbk{16|S||A|k^*l^*/\delta}}{(1-\gamma)^3k^*}}
\right].
\end{aligned}
\]
By \eqref{eqn:span_inner_log_adjustment}, the two terms in brackets are bounded respectively by universal constants times $(1-\gamma)^2/C$ and $1/\sqrt C$. Increasing $C$ if necessary makes the entire contraction factor at most $1/4$. It follows that
\[
\norminf{q_{l,k^*+1}^c-\bar q_l}
\leq\frac14\norminf{\hat q_{l-1}^c-\bar q_l}.
\]

Therefore, combining this with \eqref{eqn:tu_q_span_err_shift_equiv} and the choice of $d_l$ so that \eqref{eqn:tu_choice_of_dl} holds, we conclude that, with $\cH_l$-conditional probability at least $1-\delta/(2l^*)$,
\begin{equation}
\label{eqn:inner_span_contraction}
\spnorm{\hat q_l-\bar q_l}\leq 2 \norminf{q_{l,k^*+1}^c-\bar q_l}
\leq\frac12\norminf{\hat q_{l-1}^c-\bar q_l} \leq\frac14\spnorm{\hat q_{l-1}-\bar q_l}
\end{equation}
where the first inequality bounds the span by two times the sup norm. 

\paragraph{Union bound over the epochs.}
Note that on the intersection of \eqref{eqn:bar_q_span_epoch_bound} and \eqref{eqn:inner_span_contraction},
\begin{align}
\spnorm{\hat q_l-q^*}
&\leq\spnorm{\hat q_l-\bar q_l}+\spnorm{\bar q_l-q^*}\notag\\
&\leq\frac14\crbk{\spnorm{\hat q_{l-1}-q^*}+\spnorm{q^*-\bar q_l}}+\spnorm{\bar q_l-q^*}\notag\\
&\leq\frac14\crbk{2b_l+\frac25b_l}+\frac25b_l\notag\\
&=b_l.
\label{eqn:span_good_event_vrql}
\end{align}

Define
\[
B_l:=\bigcap_{j=0}^l\set{\spnorm{\hat q_j-q^*}\leq b_j}.
\]
The event $B_{l-1}$ is $\cG_{l-1}$-measurable. Conditional on $\cG_{l-1}$ and for sample paths in $B_{l-1}$, the preceding fixed-point argument shows that
\[
P\crbkcond{\spnorm{\bar q_l-q^*}>\frac25b_l}{\cG_{l-1}}
\leq\frac{\delta}{2l^*}.
\]
Therefore,
\[
P\crbk{\set{\spnorm{\bar q_l-q^*}>\frac25b_l}\cap B_{l-1}}=E\left[\1_{B_{l-1}}P\crbkcond{\spnorm{\bar q_l-q^*}>\frac25b_l}{\cG_{l-1}}\right]\leq\frac{\delta}{2l^*}.
\]
The event $\set{\spnorm{\bar q_l-q^*}\leq2b_l/5}\cap B_{l-1}$ is $\cH_l$-measurable. Conditional on $\cH_l$, the inner-loop contraction \eqref{eqn:inner_span_contraction} fails with probability at most $\delta/(2l^*)$. Hence,
\[
\begin{aligned}
&P\crbk{\set{\spnorm{\hat q_l-\bar q_l}>\frac14\spnorm{\hat q_{l-1}-\bar q_l}}\cap\set{\spnorm{\bar q_l-q^*}\leq\frac25b_l}\cap B_{l-1}}\\
&\quad=E\left[
\1_{\set{\spnorm{\bar q_l-q^*}\leq2b_l/5}\cap B_{l-1}}
P\crbkcond{\spnorm{\hat q_l-\bar q_l}>\frac14\spnorm{\hat q_{l-1}-\bar q_l}}{\cH_l}
\right]\\
&\quad\leq\frac{\delta}{2l^*}.
\end{aligned}
\]
Whenever $B_{l-1}$ holds and neither of these two failures occurs, \eqref{eqn:span_good_event_vrql} gives $\spnorm{\hat q_l-q^*}\leq b_l$. Thus,
\[
P\crbk{B_l^c\cap B_{l-1}}
\leq\frac{\delta}{l^*},
\]
and hence $P(B_l^c)
\leq P(B_{l-1}^c)+\delta/l^*.$
Since $P(B_0)=1$, induction gives $P(B_{l^*}^c)\leq\delta.$
This proves \eqref{eqn:pathwise_span_q_error}.
\end{proof}

\subsection{Proof of Lemma~\ref{lemma:span_recentered_perturbation_bound}}
\label{subsection:proof_span_recentered_perturbation_bound}

\begin{proof}
Under \eqref{eqn:span_recenter_concentration}, we may decompose
\[
z_l=h_1+h_2,
\]
where, entry-wise,
\[
\abs{h_1}
\leq
\sqrt{\frac{2\log\crbk{8|S||A|l^*/\delta}}{n_l}}\,
\sigma(\hat q_{l-1})
\]
and
\[
\norminf{h_2}
\leq
\frac{2\gamma\spnorm{\hat q_{l-1}}
\log\crbk{8|S||A|l^*/\delta}}{3n_l}.
\]

Using the triangle inequality and homogeneity of the span seminorm, followed by the two assertions of Proposition~\ref{prop:signed_propagated_std_uniform_mixing}, gives
\[
\begin{aligned}
\spnorm{(I-\gamma P^{\pi^\circ})\inv z_l}
&\leq
\spnorm{(I-\gamma P^{\pi^\circ})\inv h_1}
+
\spnorm{(I-\gamma P^{\pi^\circ})\inv h_2}\\
&\leq
\sqrt{\frac{2\log\crbk{8|S||A|l^*/\delta}}{n_l}}
\sup_{\abs{h}\leq\sigma(\hat q_{l-1})}
\spnorm{(I-\gamma P^{\pi^\circ})\inv h}\\
&\quad+
\frac{2\gamma\spnorm{\hat q_{l-1}}
\log\crbk{8|S||A|l^*/\delta}}{3n_l}
\sup_{\norminf{h}\leq1}
\spnorm{(I-\gamma P^{\pi^\circ})\inv h}\\
&\leq
\sqrt{\frac{2\log\crbk{8|S||A|l^*/\delta}}{n_l}}
\left[
8\sqrt6\spnorm{r^\circ}\tmix^{3/2}
+
4\tmix\spnorm{\hat q_{l-1}-q^\circ}
\right]\\
&\quad+
\frac{16\gamma\tmix\spnorm{\hat q_{l-1}}
\log\crbk{8|S||A|l^*/\delta}}{3n_l}.
\end{aligned}
\]
Since $\spnorm{r^\circ}\leq5$, $\spnorm{\hat q_{l-1}}\leq16\tmix$, and $\gamma\leq1$, it follows that
\[
\begin{aligned}
\spnorm{(I-\gamma P^{\pi^\circ})\inv z_l}
&\leq
80\sqrt{3}
\left[
\tmix^{3/2}
+
\tmix\spnorm{\hat q_{l-1}-q^\circ}
\right]
\sqrt{\frac{\log\crbk{8|S||A|l^*/\delta}}{n_l}}\\
&\quad+
\frac{256}{3}
\frac{\tmix^2\log\crbk{8|S||A|l^*/\delta}}{n_l}.
\end{aligned}
\]
Thus, Lemma \ref{lemma:span_recentered_perturbation_bound} holds with $C_1=80\sqrt{3}$ and $C_2=\frac{256}{3}.$
\end{proof}

\end{document}